\documentclass[10pt,journal,compsoc]{IEEEtran}

\usepackage{ifpdf}
\usepackage{cite}
\usepackage[pdftex]{graphicx}
\usepackage[cmex10]{amsmath}
\usepackage{times}
\usepackage{amsmath}
\usepackage{bbm}
\usepackage{amssymb}
\usepackage{algorithm} 
\usepackage{algorithmic} 
\usepackage{graphics}
\usepackage{color}
\usepackage{subfigure}
\usepackage{epsfig}
\usepackage{amsthm}
\usepackage{epstopdf}

\usepackage{geometry}

\newtheorem{thm}{Theorem}
\newtheorem{lem}[thm]{Lemma}

\DeclareMathOperator*{\argmax}{arg\,max}

\DeclareMathSizes{10}{9}{7}{5}
\DeclareMathSizes{11}{10}{7.7}{5.5}

\newcommand\ws[1]{{\color{black}#1}} 
\newcommand\lk[1]{{\color{black}#1}}

\newcommand\red[1]{{\color{black}#1}}

\begin{document}


\newgeometry{top=6cm,bottom=1cm}  
\onecolumn{

 
\noindent \textbf{\Huge{Online Hashing}}

\vspace{2cm}

\noindent {\LARGE{Long-Kai Huang, Qiang Yang, Wei-Shi Zheng}}

\Large
\vspace{2cm}

\noindent Code is available at {http://isee.sysu.edu.cn/\%7ezhwshi/code/OLHash\%5fcode.zip}


\vspace{1cm}

\noindent For reference of this work, please cite:

\vspace{1cm}
\noindent Long-Kai Huang, Qiang Yang, Wei-Shi Zheng. Online Hashing.  IEEE Transactions on Neural Networks and Learning Systems (DOI: 10.1109/TNNLS.2017.2689242)

\vspace{1cm}

\noindent Bib:
\noindent 
@article\{onlinehashing,\\
\ \ \   title=\{Online Hashing\},\\
\ \ \  author=\{Long-Kai Huang and Qiang Yang and Wei-Shi Zheng\},\\
\ \ \  journal=\{IEEE Transactions on Neural Networks and Learning Systems (DOI: 10.1109/TNNLS.2017.2689242)\}\\
\}

}

 \clearpage  
  
\restoregeometry

\title{Online Hashing}
\author{

Long-Kai Huang, Qiang Yang, Wei-Shi Zheng
\IEEEcompsocitemizethanks{
\IEEEcompsocthanksitem 

L.-K. Huang is with the School of Data and Computer Science, Sun
Yat-sen University, Guangzhou 510275, China and also with School of Computer Science and Engineering, Nanyang Technological University, Singapre.
Email: hlongkai@gmail.com.

\protect
\IEEEcompsocthanksitem 

Q. Yang is with School of Data and Computer Science, Sun Yat-sen
University, Guangzhou, 510275, China.
Email: mmmyqmmm@gmail.com.
\protect
\IEEEcompsocthanksitem 

Wei-Shi Zheng is with the School of Data and Computer Science, Sun
Yat-sen University, Guangzhou 510275, China, and is also with the Key
Laboratory of Machine Intelligence and Advanced Computing (Sun Yatsen
University), Ministry of Education, China. 
E-mail: wszheng@ieee.org.

}
}



\IEEEcompsoctitleabstractindextext{
\begin{abstract}

Although hash function learning algorithms have achieved great success in recent years, most existing hash models are off-line, which are not suitable for processing sequential or online data. To address this problem, this work proposes an online hash model to accommodate data coming in stream for online learning. Specifically, a new loss function is proposed to measure the similarity loss between a pair of data samples in hamming space. Then, a structured hash model is derived and optimized in a passive-aggressive way. Theoretical analysis on the upper bound of the cumulative loss for the proposed online hash model is provided.
{Furthermore, we extend our online hashing from a single-model to a multi-model online hashing that trains multiple models so as to retain diverse online hashing models in order to avoid biased update.}
The competitive efficiency and effectiveness of the proposed online hash models are verified through extensive experiments on several large-scale datasets as compared to related hashing methods.
\end{abstract}

}

\maketitle

\section{Introduction}

There have been great interests in representing data using compact binary codes in recent developments. Compact binary codes not only facilitate storage of large-scale data but also benefit fast similarity computation, so that they are applied to fast nearest neighbor search~\cite{webb2003statistical,Xing_SideConstraint_2003a,Weinberger_MarginLocal_2006a,Zheng_RDC_2012a}, as it only takes very short time (generally less than a second) to compare a query with millions of data points~\cite{multiindex}.


For learning compact binary codes, a number of hash function learning algorithms have been developed in the last five years. There are two types of hashing methods: the data independent ones and the data dependent ones.
Typical data independent hash models include 
Locality Sensitive Hashing (LSH)~\cite{Charikar:LSH} and its variants like $\ell_p$-stable hashing \cite{Datar:LSH_Lp}, min-hash \cite{Chum:LSH_Min_hash} and kernel LSH (KLSH)\cite{Kulis:KLSH}. \ws{Since using information of data distribution or class labels would make significant improvement in fast search, }more efforts are devoted to the data-dependent approach
\cite{robust-hashing2014,Topology-Preserving-Hashing2014,one-permutation-hashing2014,List-Wise-Hashing2014,kernel-hashing2014,QuanCorrH,SupDiscH}. For the data dependent hashing methods, 
they are categorized into unsupervised-based~\cite{Weiss:SH, Liu:GH, Gong:ITQ, Heo:SphericalH}, supervised-based~\cite{ Liu:KSH, Norouzi:Min_Loss_Hash, Mu:WSH,supervised-hashing20141,supervised-hashing20142,supervised-hashing20143}, and semi-supervised-based~\cite{Wang:SSH, Wang:S3PLH,semi-supervised-hashing2014} hash models. In addition to these works, multi-view hashing~\cite{multiview_pp,multiview_dual}, multi-modal hashing~\cite{Zhang_MultimodalScemantic,Masci_MultimodalSP,Heterogeneous-Hashing2014,MultiModel15_xinboGao,MultiModel16_xinboGao}, and active hashing~\cite{active-hashing,active-hashing2014} have also been developed.

In the development of hash models, a challenge remained unsolved is that most hash models are learned in an offline mode or batch mode, that is to assume all data are available in advance for learning the hash function. However, learning hash functions with such an assumption has the following critical limitations:
\begin{itemize}
  \item First, they are hard to be trained on very large-scale training datasets, since they have to make all learning data kept in the memory, which is costly for processing. Even though the memory is enough, the training time of these methods on large-scale datasets is intolerable. With the advent of big data, these limitations become more and more urgent to solve.
  \item Second, they cannot adapt to sequential data or new coming data. In real life, data \lk{samples} are usually collected sequentially as time passes, and some early collected data may be outdated. When the differences between the already collected data and the new coming data are large, current hashing methods usually lose their efficiency on the new data samples. Hence, it is important to develop online hash models, which can be efficiently updated to deal with sequential data.
\end{itemize}

In this paper, we overcome the limitations of batch mode hash methods \cite{Weiss:SH, Liu:GH, Gong:ITQ,Liu:KSH, Mu:WSH,supervised-hashing20141,supervised-hashing20142,supervised-hashing20143} by developing an effective online hashing learning method called \emph{Online Hashing} (OH).
We propose a one-pass online adaptation criterion in a passive-aggressive way~\cite{ Crammer:PA}, which enables the newly learned hash model to embrace information from a new pair of data samples in the current round and meanwhile retain important information learned in the previous rounds. In addition, the exact labels of data are not required, and only the pairwise similarity is needed.
More specifically, a similarity loss function is first designed to measure the confidence of the similarity between two hash codes of a pair of data samples, and then based on that similarity loss function a prediction loss function is proposed to evaluate whether the current hash model fits the current data under a structured prediction framework. We then minimize the proposed prediction loss function on the current input pair of data samples to update the hash model.
During the online update, we wish to make the updated hash model approximate the model learned in the last round
as much as possible for retaining the most historical discriminant information during the update.
An upper bound on the cumulative similarity loss of the proposed online algorithm are derived,
so that the performance of our online hash function learning can be guaranteed.



\ws{Since one-pass online learning only relies on the new data at the current round, the adaptation could be easily biased by the current round data. Hence, we introduce a multi-model online strategy in order to alleviate such a kind of bias, where multiple but not just one online hashing models are learned and they are expected to suit more diverse data pairs and will be selectively updated. A theoretical bound on the cumulative similarity loss is also provided.}

In summary, the contributions of this work are
\begin{itemize}
  \item[1)] Developing
  a weakly supervised online hash function learning model. In our development, a novel similarity loss function is proposed to measure the difference of the hash codes of a pair of data samples in Hamming space.
  Following the similarity loss function is the prediction loss function to penalize the violation of the given similarity between the hash codes in Hamming space. Detailed theoretical analysis is presented to give a theoretical upper loss bound for the proposed online hashing method;
  \item[2)] Developing a Multi-Model Online Hashing (MMOH), in which 
a multi-model similarity loss function is proposed to guide the training of multiple complementary hash models.
\end{itemize}

The rest of the paper is organized as follows. In Sec. \ref{sec:RelatedWork}, related literatures are reviewed. In Sec. \ref{sec:II}, we present our online algorithm framework including the optimization method. Sec. \ref{sec_gain_g} further elaborates one left issue in Sec. \ref{sec:II} for acquiring zero-loss binary codes. Then we give analysis on the upper bound and time complexity of OH in Sec. \ref{sec:analysis}, and extend our algorithm to a multi-model one in Sec. \ref{sec:MMOH}. Experimental results for evaluation are reported in Sec. \ref{sec:Experiment} and finally we conclude the work in Sec. \ref{sec:conclusion}.

\section{Related Work}\label{sec:RelatedWork}

Online learning, especially one-pass online learning, plays an important role for processing large-scale datasets, as it is time and space efficient. It is able to learn a model based on streaming data, making dynamic update possible. In typical one-pass online learning algorithms~\cite{ Crammer:PA, Chechik:Image_Similarity_Through_Ranking}, when an instance is received, the algorithm makes a prediction, receives the feedback, and then updates the model based on this new data sample only upon the feedback. Generally, the performance of an online algorithm is guaranteed by the upper loss bound in the worst case.

There are a lot of existing works of online algorithms to solve specific machine learning problems~\cite{Chechik:Image_Similarity_Through_Ranking, Jain:Fast_Online_Similarity_Search, Li:ROMMA, OnlinePCA,Crammer:PA}. However, it is difficult to apply these online methods to online hash function learning, because the sign function used in hash models is non-differentiable, which makes the optimization problem more difficult to solve.
Although one can replace the sign function with sigmoid type functions or other approximate differentiable functions and then apply gradient descent, this becomes an obstacle on deriving the loss bound. There are existing works considering active learning and online learning together \cite{silva2012active,bordes2005fast,chu2011unbiased}, but they are not for hash function learning.

Although it is challenging to design hash models in an online learning mode, several hashing methods are related to
online learning~\cite{Jain:Fast_Online_Similarity_Search, Norouzi:Min_Loss_Hash,OSH,ArxivStream,OnlineSupervisedHashing,AdaptHash}. In \cite{Jain:Fast_Online_Similarity_Search}, the authors realized an online LSH by applying an online metric learning algorithm, namely LogDet Exact Gradient Online (LEGO) to
LSH. Since \cite{Jain:Fast_Online_Similarity_Search} is operated on LSH, which is a data independent hash model, it does not directly optimize hash functions for generating compact binary code in an online way. \ws{The other five related works can be categorized into two groups: one is the stochastic gradient descent (SGD) based online methods, including minimal loss hashing \cite{Norouzi:Min_Loss_Hash}, online supervised hashing \cite{OnlineSupervisedHashing} and Adaptive Hashing (AdaptHash) \cite{AdaptHash}; another group is matrix sketch based methods, including online sketching hashing (OSH)~\cite{OSH} and stream spectral binary coding (SSBC) \cite{ArxivStream}.}

Minimal loss hashing(MLH) \cite{Norouzi:Min_Loss_Hash} follows the loss-adjusted inference used in structured SVMs and deduces the convex-concave upper bound on the loss. Since MLH is a hash model relying on stochastic gradient decent update for optimization, it can naturally be used for online update.
However, there are several limitations that make MLH unsuitable for online processing. First, the upper loss bound derived by MLH is actually related to the number of the historical samples used from the beginning. In other words, the upper loss bound of MLH may grow as the number of samples increases and therefore its online performance could not be guaranteed. Second, MLH assumes that all input data are centered (i.e. with zero mean), but such a pre-processing is challenging for online learning since all data samples are not available in advance.

\ws{

Online supervised hashing (OECC)\cite{OnlineSupervisedHashing} is a SGD version of the Supervised Hashing with Error Correcting Codes (ECC) algorithm \cite{ECC}. 
It employs a 0-1 loss function 
which outputs either 1 or 0 to indicate whether the binary code generated by existing hash model is in the codebook generated by error correcting output codes algorithm. If it is not in the codebook, the loss is 1. 
After replacing the 0-1 loss with a convex loss function and dropping the non-differentiable sign function in the hash function, SGD is applied to minimize the loss and update the hash model online. 
AdaptHash \cite{AdaptHash} is also a SGD based methods. It defines a loss function the same as the hingle-like loss function used in \cite{ Norouzi:Min_Loss_Hash, Huang:OLHashing}. To minimize this loss, the authors approximated the hash function by a differentiable sigmoid function and then used SGD to optimize the problem in an online mode. 
Both OECC and AdaptHash do not assume that data samples have zero mean as used in MLH. They handle the zero-mean issue by a method similar to the one in \cite{ Huang:OLHashing}. All these three SGD-based hashing methods enable online update by applying SGD, but they all cannot guarantee a constant loss upper bound. }

Online sketching hashing (OSH)~\cite{OSH} was recently proposed to enable learning hash model on stream data
by combining PCA hashing~\cite{Wang:SSH} and matrix sketching~\cite{MatrixSketch}. It first sketches stream samples into a small size matrix and meanwhile guarantee approximating data covariance, and then PCA hashing can be applied on this sketch matrix to learn a hash model. Sketching overcomes the challenge of training a PCA-based model on sequential data using limited storage. Stream spectral binary coding (SSBC) \cite{ArxivStream} is another learning to hash method on stream data based on the matrix sketch~\cite{MatrixSketch} skill. It applies matrix sketch processing on the Gaussian affinity matrix in spectral hashing algorithm\cite{ Weiss:SH}. Since the sketched matrix reserves global information of previously observed data, the new update model may not be adapted well on new observed data samples after a large number of samples have been sketched in the previous steps.

A preliminary version of this work was presented in~\cite{Huang:OLHashing}. Apart from more detailed analysis, this submission differs from the conference version in the following aspects: 1) we have further modified the similarity loss function to evaluate the hash model and to guide the update of the hash model; 

2) we have developed a multi-model strategy to train multiple models to improve online hashing; 
3) We have modified the strategy of zero-loss codes inference, which suits the passive aggressive scheme better;
4) much more experiments have been conducted to demonstrate the effectiveness of our method.

\section{Online Modelling for Hash Function Learning}\label{sec:II}

Hash models aim to learn a model to map a given data vector $\mathbf{x} \in \mathbb{R}^d$ to an $r$-bit binary code vector $\mathbf{h} \in \{-1,1\}^r$.
For $r$-bit linear projection hashing, the \emph{k}$^{th} (1\le k \le r)$ hash function is defined as
\begin{displaymath}
\begin{split}
h_k(\mathbf{x}) &= sgn(\mathbf{w}_k^T\mathbf{x}+ b_k)= \left\{
\begin{aligned}
1 &,& &if \ \ \mathbf{w}_k^T\mathbf{x}+ b_k \ge 0, \\
-1 &,& &otherwise,
\end{aligned}
\right.
\end{split}
\end{displaymath}
where $\mathbf{w}_k\in \mathbb{R}^d$ is a projection vector, $b_k$ is a scalar threshold, and $h_k$ $\in \{-1,1\}$ is the binary hash code.

Regarding 
$b_k$ in the above equation, a useful guideline proposed in \cite{Weiss:SH,Wang:SSH,Xu:CH} is that the scalar threshold $b_k$ should be the median of $\{ \mathbf{w}_k^T \mathbf{x_i} \}_{i=1}^n$, where $n$ is the number of the whole input samples, in order to guarantee that half of the output codes are $1$ while the other half are $-1$.
This guarantees achieving maximal information entropy of every hash bit \cite{Weiss:SH,Wang:SSH,Xu:CH}. A relaxation is to set $b_k$ to the mean of $\{ \mathbf{w}_k^T \mathbf{x_i} \}_{i=1}^n$ and $b_k$ will be zero if the data are zero-centered. Such a pretreatment not only helps to improve performance but also simplifies the hash function learning.
Since data come in sequence, it is impossible to obtain the mean in advance. In our online algorithm, we will estimate the mean after a few data samples are available, update it after new data samples arrive, and perform update of zero-centered operation afterwards.
Hence, the $r$-bit hash function becomes
\begin{equation}\label{eq:hash_function}
\mathbf{h} = h(\mathbf{x}) = sgn(\mathbf{W}^T\mathbf{x}), 
\end{equation}
where $\mathbf{W}=[\mathbf{w_1}, \mathbf{w_2}, ... , \mathbf{w_r}] \in \mathbb{R}^{d\times r}$ is the hash projection matrix and $\mathbf{h}=[h_1(\mathbf{x}), h_2(\mathbf{x}), ... , h_r(\mathbf{x})]^T$ is the r-bit hash code. Here, $\mathbf{x}$ is the data point after zero-mean shifting.

Unfortunately, due to the non-differentiability of the sign function in Eq. (\ref{eq:hash_function}), the optimization of the hash model becomes difficult. To settle such a problem, by borrowing the ideas from the structured prediction in structured SVMs~\cite{Finley:SSVM, Structured_Output} and MLH~\cite{Norouzi:Min_Loss_Hash}, we transform Eq. (\ref{eq:hash_function}) equivalently to the following structured prediction form:
\begin{equation}\label{eq:structured prediction hash function}
\mathbf{h} = \argmax_{\mathbf{f}\in \{-1,1\}^r} \mathbf{f}^T\mathbf{W}^T \mathbf{x}. 
\end{equation}
In the structured hash function Eq. (\ref{eq:structured prediction hash function}), $\mathbf{f}^T\mathbf{W}^T \mathbf{x}$ can be regarded as a prediction value that measures the extent that the binary code $\mathbf{f}$ matches the hash code of $\mathbf{x}$ obtained through Eq. (\ref{eq:hash_function}).
Obviously, $\mathbf{f}^T\mathbf{W}^T \mathbf{x}$ is maximized only when each element of $\mathbf{f}$ has the same sign as that of $\mathbf{W}^T \mathbf{x}$.

\subsection{Formulation}\label{sec:Ol_algorithm}

We start presenting the online hash function learning. In this work, we assume that the sequential data are in pairs. Suppose a new pair of data points $\mathbf{x_i}^t$ and $\mathbf{x_j}^t$ comes in the $t^{th} \ (t=1,2,...)$ round with a similarity label $s^t$. The label indicates whether these two points are similar or not and it is defined as:
\begin{displaymath}
s^t =
\left\{
\begin{array}{ll}
1, &  if \ \mathbf{x_i}^t \ \text{and} \  \mathbf{x_j}^t \text{ are similar}, \\
-1, & if \ \mathbf{x_i}^t \ \text{and} \  \mathbf{x_j}^t \text{ are not similar}.
\end{array}
\right.
\end{displaymath}

We denote the new coming pair of data points $\mathbf{x_i}^t$ and $\mathbf{x_j}^t$ by $\mathbbm{x}^t = [\mathbf{x_i}^t,\mathbf{x_j}^t]$.
In the $t^{th}$ round, based on the hash projection matrix $\mathbf{W}^{t}$ learned in the last round, we can compute the hash codes of $\mathbf{x_i}^t$ and $\mathbf{x_j}^t$ denoted by $\mathbf{h_i}^t$, $\mathbf{h_j}^t$, respectively. Such a pair of hash codes is denoted by $\mathbbm{h}^t=[\mathbf{h_i}^t, \mathbf{h_j}^t]$.

However, the prediction does not always match the given similarity label information of $\mathbbm{x}^t$. When a mismatch case is observed, $\mathbf{W}^{t}$, the model learned in the ${(t-1)}^{th}$ round, needs to be updated for obtaining a better hash model $\mathbf{W}^{t+1}$. In this work, the Hamming distance is employed to measure the match or mismatch cases.
In order to learn an optimal hash model that minimizes the loss caused by mismatch cases and maximizes the confidence
of match cases, we first design a similarity loss function to quantify the difference between the pairwise hash codes $\mathbbm{h}^t$ with respect to the corresponding similarity label $s^t$, which is formulated as follows:
\begin{equation}\label{eq:loss_function}
R(\mathbbm{h}^t,s^t)=
\left\{
\begin{array}{ll}
max\{0, \mathcal{D}_h(\mathbf{h_i}^t, \mathbf{h_j}^t) - \alpha \},  &if \ s^t=1, 
\\
max\{0, \beta  r -\mathcal{D}_h(\mathbf{h_i}^t, \mathbf{h_j}^t)\},  &if \ s^t=-1,  
\\
\end{array}
\right.
\end{equation}
where $\mathcal{D}_h(\mathbf{h_i}^t, \mathbf{h_j}^t)$ is the Hamming distance between $\mathbf{h_i}^t$ and $\mathbf{h_j}^t$, $\alpha $ is an integer playing as the similarity threshold, and $\beta$ is the dissimilarity ratio threshold ranging from $0$ to $1$. Generally, $\beta r > \alpha$, so that there is a certain margin between the match and mismatch cases.

\begin{figure}[t]
\centering {
\includegraphics[height=0.2 \linewidth]{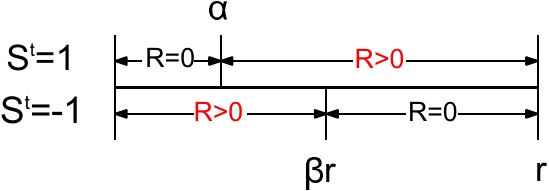}
\vspace{-0.4cm}
\caption{\small Similarity loss function. The top is for a pair of similar data samples, and the bottom is for a pair of dissimilar ones.} 
\label{fig:loss_func}}
\vspace{-0.3cm}
\end{figure}

In the above loss function, a relaxation is actually introduced by employing the threshold parameters $\alpha$ and $\beta r$ as shown in Fig. \ref{fig:loss_func}. $\alpha$ should not be too large, so that the similarity of the pairwise data samples can be preserved in the Hamming space. In contrast, $\beta r$ should not be too small; otherwise two dissimilar data points cannot be well separated in Hamming space.
From Fig. \ref{fig:loss_func}, we can see that the mismatch can be one of the following two cases: 1) the Hamming distance between the prediction codes $\mathbf{h_i}^t$ and $\mathbf{h_j}^t$ is larger than $\alpha$ for a similar pair; and 2) the Hamming distance between the prediction codes $\mathbf{h_i}^t$ and $\mathbf{h_j}^t$ is smaller than $\beta r$ for a dissimilar pair. These two measure the risk of utilizing the already learned hash projection matrix $\mathbf{W}^{t}$ on a new pair of data points, i.e. $R(\mathbbm{h}^t,s^t)$.


If the model learned in the last round predicts zero-loss hash code pair on a new pair, that is $R(\mathbbm{h}^t,s^t)=0$, our strategy is to retain the current model. If the model is unsatisfactory, i.e., predicting inaccurate hash codes having $R(\mathbbm{h}^t,s^t)>0$, we need to update the inaccurate previous hash projection matrix $\mathbf{W}^{t}$.

To update the hash model properly, we claim that the zero-loss hash code pair $\mathbbm{g}^t=[\mathbf{g_i}^t, \mathbf{g_j}^t]$ for current round data pair $\mathbbm{x}^t$ satisfying $R(\mathbbm{g}^t,s^t)=0$ is available, and we use the zero-loss hash code pair to guide update of the hash model, deriving the updated $\mathbf{W}^{t+1}$ towards a better prediction. We leave the details about how to obtain the zero-loss hash code pair presented in Sec. \ref{sec_gain_g}.

Now, we wish to obtain an updated hash projection matrix $\mathbf{W}^{t+1}$ such that it predicts similar hash code pair towards the zero-loss hash code pair $\mathbbm{g}^t$ for the current input pair of data samples.
Let us define
\begin{equation}\label{eq:H}
H^t(\mathbf{W}) = {\mathbf{h_i}^t}^T\mathbf{W}^T\mathbf{x_i}^t+{\mathbf{h_j}^t}^T\mathbf{W}^T\mathbf{x_j}^t, \end{equation}
\begin{equation}\label{eq:G}
G^t(\mathbf{W}) = {\mathbf{g_i}^t}^T\mathbf{W}^T\mathbf{x_i}^t+{\mathbf{g_j}^t}^T\mathbf{W}^T\mathbf{x_j}^t.
\end{equation}

Given hash function Eq. (\ref{eq:structured prediction hash function}) with respect to $\mathbf{W}^t$, we have $H^t(\mathbf{W}^t) \ge G^t(\mathbf{W}^t)$, since $\mathbf{h_i}^t$ and $\mathbf{h_j}^t$ are the binary solutions for $\mathbf{x_i}^t$ and $\mathbf{x_j}^t$ for the {maximization}, respectively, while $\mathbf{g_i}^t$ and $\mathbf{g_j}^t$ are not. \ws{This also suggests the $\mathbf{W}^t$ is not suitable for the generated binary code to approach the zero-loss hash code pair $\mathbbm{g}^t$, and thus a new projection $\mathbf{W}^{t+1}$ has to be learned.}

\ws{
When updating the projection matrix from $\mathbf{W}^{t}$ to $\mathbf{W}^{t+1}$, we expect that the binary code generated for \red{$\mathbf{x_i}^t$} is $\mathbf{g_i}^t$. According to the hash function in structured prediction form in Eq. (\ref{eq:structured prediction hash function}), our expectation is to require ${\mathbf{g_i}^t}^T{\mathbf{W}^{t+1}}^T\mathbf{x_i}^t > {\mathbf{h_i}^t}^T{\mathbf{W}^{t+1}}^T\mathbf{x_i}^t$. Similarly, we expect the binary code generated for \red{$\mathbf{x_j}^t$} is $\mathbf{g_j}^t$ and this is also to require ${\mathbf{g_j}^t}^T{\mathbf{W}^{t+1}}^T\mathbf{x_j}^t > {\mathbf{h_j}^t}^T{\mathbf{W}^{t+1}}^T\mathbf{x_j}^t$. Combining these two inequalities together, it would be expected that the new $\mathbf{W}^{t+1}$ should meet the condition that $G^t(\mathbf{W}^{t+1}) > H^t(\mathbf{W}^{t+1})$.}
To achieve this objective, we derive the following
prediction loss function $\ell^t(\mathbf{W})$ for our algorithm:
\begin{equation}\label{eq:prediction-based loss}
\begin{split}
	\ell^t(\mathbf{W}) = H^t(\mathbf{W}) - G^t(\mathbf{W}) + \sqrt{R(\mathbbm{h}^t,s^t)}.
\end{split}
\end{equation}
In the above loss function, $\mathbbm{h}^t$, $\mathbbm{g}^t$ and $R(\mathbbm{h}^t,s^t)$ are constants rather than variables dependent on $\mathbf{W}^{t}$. $R(\mathbbm{h}^t,s^t)$ can be treated a loss penalization. When used in the Criterion (\ref{eq:optimization-formulation}) later, a small $R(\mathbbm{h}^t,s^t)$ means a slight update is expected, and 
a large $R(\mathbbm{h}^t,s^t)$ means a large update is necessary. Note that the square root of similarity loss function $R(\mathbbm{h}^t,s^t)$ is utilized here, because it enables an upper bound on the cumulative loss functions, which will be shown in Sec.\ref{sec:Loss_Bound}.

Note that if $\ell^t(\mathbf{W}^{t+1})=0$, we can have $G^t(\mathbf{W}^{t+1}) = H^t(\mathbf{W}^{t+1}) + \sqrt{R(\mathbbm{h}^t,s^t)} > H^t(\mathbf{W}^{t+1})$.
Let $\widehat{\mathbbm{g}}^t$ be the hash codes of $\mathbbm{x}^t$ computed using the updated $\mathbf{W}^{t+1}$ by Eq. (\ref{eq:structured prediction hash function}). Even though $G^t(\mathbf{W}^{t+1}) > H^t(\mathbf{W}^{t+1})$ cannot guarantee that $\widehat{\mathbbm{g}}^t$ is exactly $\mathbbm{g}^t$, it is probable that $\widehat{\mathbbm{g}}^t$ is very close to $\mathbbm{g}^t$ rather than $\mathbbm{h}^t$. It therefore makes sense to force $\ell^t(\mathbf{W}^{t+1})$ to be zero or close to zero.

Since we are formulating a one-pass learning algorithm, 
the previously observed data points are not available for the learning in the current round, and the only information we can make use of is the current round projection matrix $\mathbf{W}^{t}$. In this case, we force that the newly learned $\mathbf{W}^{t+1}$ should stay close to the projection matrix $\mathbf{W}^{t}$ as much as possible so as to preserve the information learned in the last round as much as possible. Hence, the objective function for updating the hash projection matrix becomes
\begin{equation}\label{eq:optimization-formulation}
\begin{split}
& \mathbf{W}^{t+1}= \arg \min_{\mathbf{W}}\frac{1}{2}||\mathbf{W}-\mathbf{W}^{t}||^2_F + C\xi,
\\
& \quad s.t. \quad \ell^t(\mathbf{W})\le \xi
\quad\text{and}\quad \xi \ge 0,
\end{split}
\end{equation}
where $\|\cdot\|_F$ is the Frobenius norm, $\xi$ is a non-negative auxiliary variable to relax the constraint on the prediction loss function $\ell^t(\mathbf{W})=0$,
and $C$ is a margin parameter to control the effect of the slack term, whose influence will be observed in Sec.\ref{sec:Experiment}. \ws{Through this objective function, the difference between the new projection matrix $\mathbf{W}^{t+1}$ and the last one $\mathbf{W}^{t}$ is minimized, and meanwhile the prediction loss function $\ell^t(\mathbf{W})$ of the new $\mathbf{W}^{t+1}$ is bounded by a small value.} We call the above model the \emph{online hashing} (OH) model.


Finally, we wish to provide a comment on the function $H^t(\mathbf{W})$ in Eq. (\ref{eq:H}) and Eq. (\ref{eq:prediction-based loss}). Actually, an optimal case should be to refine function $H^t(\mathbf{W})$ as a function of variables $\mathbf{W}$ and a code pair $\mathbbm{f} = [\mathbf{f_i}, \mathbf{f_j}] \in \{-1,1\}^{r \times 2}$ as follows:
\begin{equation}
	H^t(\mathbf{W},\mathbbm{f}) = {\mathbf{f_i^t}}^T\mathbf{W}^T\mathbf{x_i}^t+{\mathbf{f_j^t}}^T\mathbf{W}^T\mathbf{x_j}^t,
\end{equation}
and then to refine the prediction loss function when an optimal update $\mathbf{W}^{t+1}$ is used:
\begin{equation}
	\ell^t(\mathbf{W}^{t+1}) = \max_{\mathbbm{f} \in \{-1,1\}^{r \times 2}} H^t(\mathbf{W}^{t+1},\mathbbm{f}) - G^t(\mathbf{W}^{t+1}) + \sqrt{R(\mathbbm{h}^t,s^t)}.
\end{equation}
The above refinement in theory can make $\max_{\mathbbm{f}} H^t(\mathbf{W}^{t+1},\mathbbm{f}) - G^t(\mathbf{W}^{t+1})$ be a more rigorous loss on approximating the zero-loss hash code pair $\mathbbm{g}^t$. But, it would be an obstacle to the optimization, since $\mathbf{W}^{t+1}$ is unknown when $\max_{\mathbbm{f}} H^t(\mathbf{W}^{t+1},\mathbbm{f})$ is computed. Hence, we avert this problem by implicitly introducing an alternating optimization by first fixing $\mathbbm{f}$ to be $\mathbbm{h}^t$, then optimizing $\mathbf{W}^{t+1}$ by Criterion (\ref{eq:optimization-formulation}), and finally predicting the best $\mathbbm{f}$ for $\max_{\mathbbm{f} \in \{-1,1\}^{r \times 2}} H^t(\mathbf{W}^{t+1},\mathbbm{f})$. This process can be iterative. Although this may be useful to further improve our online model, we do not completely follow this implicit alternating processing to learn the $\mathbf{W}^{t+1}$ iteratively. This is because data are coming in sequence and it would be demanded to process a new data pair after an update of the projection matrix $\mathbf{W}$. Hence, in our implementation, we only update $\mathbf{W}^{t+1}$ once, and we provide the bound for $R(\mathbbm{h}^t,s^t)$ under such a processing in Theorem \ref{thm:bound_loss}.

\subsection{Optimization}\label{sec:Optimization}
When $R(\mathbbm{h}^t,s^t)=0$, $\mathbbm{h}^t$ is the optimal code pair and $\mathbbm{g}^t$ is the same as $\mathbbm{h}^t$, and thus $\ell^t(\mathbf{W}^t)=0$. In this case, the solution to Criterion (\ref{eq:optimization-formulation}) is $\mathbf{W}^{t+1}=\mathbf{W}^{t}$. 
That is, when the already learned hash projection matrix $\mathbf{W}^{t}$ can correctly predict the similarity label of the new coming pair of data points $\mathbbm{x}^t$, there is no need to update the hash function.
When $R(\mathbbm{h}^t,s^t)>0$, the solution is
\begin{equation}\label{eq:solution}
\begin{split}
\mathbf{W}^{t+1} = \mathbf{W}^{t} + \tau^t \mathbbm{x}^t(\mathbbm{g}^t - \mathbbm{h}^t)^T,\\
 \ \tau^t = min
\{
C, \frac{\ell^t(\mathbf{W}^t)}
{||\mathbbm{x}^t(\mathbbm{g}^t - \mathbbm{h}^t)^T||^2_F}
\}.
\end{split}
\end{equation}

The procedure for deriving the solution formulation (Eq. (\ref{eq:solution})) when $R(\mathbbm{h}^t,s^t) > 0$ is detailed as follows.

First, the objective function (Criterion (\ref{eq:optimization-formulation})) can be rewritten below when we introduce the Lagrange multipliers:
\begin{equation}\label{eq:Lagrangian}
\begin{split}
\mathcal{L}(\mathbf{W},\tau^t,\xi,\lambda)=&\frac{||\mathbf{W}-\mathbf{W}^t||^2_F}{2}+C\xi + \tau^t(\ell^t(\mathbf{W})-\xi)-\lambda\xi,
\end{split}
\end{equation}
where $\tau^t \ge 0$ and $\lambda \ge 0$ are Lagrange multipliers.
Then, by computing $\partial \mathcal{L}/\partial \mathbf{W} =0$, $\partial \mathcal{L}/\partial \xi =0$, and $\partial \mathcal{L}/\partial {\tau ^t} =0$, we can have
\begin{equation}\label{eq:derivative of L to W}
\begin{split}
&\ \ \ \ \ \ \ \ \ \ 0=\frac{\partial \mathcal{L}}{\partial \mathbf{W}}, \\
& \Rightarrow \mathbf{W} = \mathbf{W}^t + \tau^t(\mathbf{x_i}^t(\mathbf{g_i}^t - \mathbf{h_i}^t)^T + \mathbf{x_j}^t(\mathbf{g_j}^t - \mathbf{h_j}^t)^T)  \\
&  \ \ \ \ \ \ \ = \mathbf{W}^t + \tau^t\mathbbm{x}^t(\mathbbm{g}^t - \mathbbm{h}^t)^T,
\end{split}
\end{equation}

\begin{equation}\label{eq:derivative of L to xi}
\begin{split}
&\ \ \ \ \ \ 0 = \frac{\partial \mathcal{L}}{\partial \xi}= C-\tau^t-\lambda, \\
&\Rightarrow \ \tau^t = C - \lambda.
\end{split}
\end{equation}
Since $\lambda > 0$, we have $\tau^t < C$. By putting Eq. (\ref{eq:derivative of L to W}) and Eq. (\ref{eq:derivative of L to xi}) back into Eq. (\ref{eq:prediction-based loss}), we obtain
\begin{equation}\label{eq:R_W_t+1}
\ell^t(\mathbf{W}) =  - {\tau ^t}||\mathbbm{x}^t(\mathbbm{g}^t - \mathbbm{h}^t)^T||^2_F + \ell^t(\mathbf{W}^{t}).
\end{equation}
Also, by putting Eqs. (\ref{eq:derivative of L to W}), (\ref{eq:derivative of L to xi}) and (\ref{eq:R_W_t+1}) back into Eq. (\ref{eq:Lagrangian}), we have
\begin{displaymath}
\begin{split}
\mathcal{L}(\tau^t)
&= -\frac{1}{2}\ {\tau ^t}^2||\mathbbm{x}^t(\mathbbm{g}^t - \mathbbm{h}^t)^T||^2_F +\tau^t \ell^t(\mathbf{W}^t).
\end{split}
\end{displaymath}
By taking the derivative of $\mathcal{L}$ with respect to $\tau^t$ and setting it to zero, we get
\begin{equation}\label{eq:derivative of L to tau}
\begin{split}
&0=\frac{\partial \mathcal{L}}{\partial \tau^t} =-\tau^t||\mathbbm{x}^t(\mathbbm{g}^t - \mathbbm{h}^t)^T||^2_F + \ell^t(\mathbf{W}^t), \\
&\ \Rightarrow \ \ \tau^t = \frac{\ell^t(\mathbf{W}^t)}{||\mathbbm{x}^t(\mathbbm{g}^t - \mathbbm{h}^t)^T||^2_F}.
\end{split}
\end{equation}
Since $\tau^t < C$, we can obtain
\begin{equation}\label{eq:tau}
\tau^t = min
\{
C,\ \frac{\ell^t(\mathbf{W}^t)}
{||\mathbbm{x}^t(\mathbbm{g}^t - \mathbbm{h}^t)^T||^2_F}
\}.
\end{equation}

In summary, the solution to the optimization problem in Criterion (\ref{eq:optimization-formulation}) is Eq. (\ref{eq:solution}), and the whole procedure of the proposed OH is presented in Algorithm \ref{alg:OLhash}.

\begin{algorithm}[t]
\caption{Online Hashing}
\footnotesize
\label{alg:OLhash}
\begin{algorithmic}
\STATE \textbf{INITIALIZE} $\mathbf{W}^1$
\FOR {$t$ = 1,2,...}
    \STATE Receive a pairwise instance $\mathbbm{x}^t$ and similarity label $s^t$;
    \STATE Compute the hash code pair $\mathbbm{h}^t$ of $\mathbbm{x}^t$ by Eq. (\ref{eq:hash_function});
    \STATE Compute the similarity loss $R(\mathbbm{h}^t,s^t)$ by Eq. (\ref{eq:loss_function});
    \IF {$R(\mathbbm{h}^t,s^t) > 0$}
        \STATE Get the zero-loss code pair $\mathbbm{g}^t$ that makes $R(\mathbbm{g}^t, s^t)=0$;
        \STATE Compute the prediction loss $\ell^t(\mathbf{W}^t) $ by Eq. (\ref{eq:prediction-based loss});
        \STATE Set $\tau^t = min\{C, \frac{\ell^t(\mathbf{W}^t)} {||\mathbbm{x}^t(\mathbbm{g}^t - \mathbbm{h}^t)^T||^2_F} \}$;
        \STATE Update $\mathbf{W}^{t+1} = \mathbf{W}^t + \tau^t\mathbbm{x}^t(\mathbbm{g}^t - \mathbbm{h}^t)^T $;
    \ELSE
        \STATE $\mathbf{W}^{t+1} = \mathbf{W}^t $;
    \ENDIF
\ENDFOR
\end{algorithmic}
\end{algorithm}

\subsection{Kernelization} \label{section:kernel}

Kernel trick is well-known to make machine learning models better adapted to nonlinearly separable data~\cite{ Liu:KSH}. In this context, a kernel-based OH is generated by employing explicit kernel mapping to cope with the nonlinear modeling.
In details, we aim at mapping data in the original space $\mathbb{R}^d$ into a feature space $\mathbb{R}^m$ through a kernel function based on $m \ (m<d)$ anchor points, and therefore we have a new representation of $\mathbf{x}$ which can be formulated as follows:
\begin{displaymath}
z(\mathbf{x}) = [\kappa(\mathbf{x},\mathbf{x}_{(1)}),\kappa(\mathbf{x},\mathbf{x}_{(2)}), ... ,\kappa(\mathbf{x},\mathbf{x}_{(m)})]^T,
\end{displaymath}
where $\mathbf{x}_{(1)}, \mathbf{x}_{(2)}, ..., \mathbf{x}_{(m)}$ are $m$ anchors.

For our online hash model learning, we assume that at least $m$ data points have been provided in the initial stage; otherwise, the online learning will not start until at least $m$ data points have been collected, and then these $m$ data points are considered as the $m$ anchors used in the kernel trick. 
Regarding the kernel used in this work, we employ the Gaussian RBF kernel $\kappa (\mathbf{x}, \mathbf{y}) = exp(-||\mathbf{x} - \mathbf{y}||^2 /2 \mathbf{\sigma}^2)$, where we set $\mathbf{\sigma}$ to $\mathbf{1}$ in our algorithm.

\section{Zero-loss Binary Code Pair Inference}\label{sec_gain_g}

In Sec.\ref{sec:Ol_algorithm}, we have mentioned that our online hashing algorithm relies on the zero-loss code pair $\mathbbm{g}^t=[\mathbf{g_i}^t,\mathbf{g_j}^t]$ which satisfies $R(\mathbbm{g}^t,s^t)=0$. Now, we detail how to acquire $\mathbbm{g}^t$.

\vspace{0.1cm}

\noindent \textbf{Dissimilar Case}. We first present the case for dissimilar pairs. As mentioned in Sec. \ref{sec:Ol_algorithm}, to achieve zero similarity loss, the Hamming distance between the hash codes of non-neighbors should not be smaller than $\beta r$. Therefore, we need to seek the $\mathbbm{g}^t$ such that $D_h(\mathbf{g_i}^t,\mathbf{g_j}^t) \ge {\beta r}$.
Denote the $k^{th}$ bit of $\mathbf{h_i}^t$ by $\mathbf{h_i}^t_{[k]}$, and similarly we have $\mathbf{h_j}^t_{[k]}, \mathbf{g_i}^t_{[k]}, \mathbf{g_j}^t_{[k]}$. Then $D_h(\mathbf{h_i}^t,\mathbf{h_j}^t) = \sum^r_{k=1}D_h(\mathbf{h_i}^t_{[k]},\mathbf{h_j}^t_{[k]})$, where
\begin{displaymath}
D_h(\mathbf{h_i}^t_{[k]},\mathbf{h_j}^t_{[k]}) = \left\{ \begin{array}{ll} 0, & if \quad \mathbf{h_i}^t_{[k]} = \mathbf{h_j}^t_{[k]}, \\ 1, & if \quad \mathbf{h_i}^t_{[k]} \not= \mathbf{h_j}^t_{[k]}.  \end{array} \right.
\end{displaymath}
Let $\mathcal{K}_1=\{k |D_h(\mathbf{h_i}^t_{[k]},\mathbf{h_j}^t_{[k]})=1\}$ and $\mathcal{K}_0=\{k |D_h(\mathbf{h_i}^t_{[k]},\mathbf{h_j}^t_{[k]})=0\}$. To obtain $\mathbbm{g}^t$, we first set $\mathbf{g_i}^t_{[k]}= \mathbf{h_i}^t_{[k]}$ and $\mathbf{g_j}^t_{[k]}= \mathbf{h_j}^t_{[k]}$ for $k \in \mathcal{K}_1$, so as to retain the Hamming distance obtained through the hash model learned in the last round. Next, in order to increase the Hamming distance, we need to make $D_h(\mathbf{g_i}^t_{[k]},\mathbf{g_j}^t_{[k]}) = 1$ for the $k \in \mathcal{K}_0$. That is, we need to set\footnote{The hash code in our algorithm is $-1$ and $1$. Note that, $\mathbf{g_i}^t_{[k]} = -\mathbf{h_i}^t_{[k]}$ means set $\mathbf{g_i}^t_{[k]}$ to be different from $\mathbf{h_i}^t_{[k]}$} either $\mathbf{g_i}^t_{[k]} = -\mathbf{h_i}^t_{[k]}$ or $\mathbf{g_j}^t_{[k]} = -\mathbf{h_j}^t_{[k]}$, for all the $k \in \mathcal{K}_0$. 
Hence, we can pick up $p$ bits whose indexes are in set $\mathcal{K}_0$ to change/update such that
\begin{equation}\label{eq:dh}
D_h(\mathbf{g_i}^t,\mathbf{g_j}^t) = D_h(\mathbf{h_i}^t,\mathbf{h_j}^t)+ p.
\end{equation}

Now the problem is how to set $p$, namely the number of hash bits to update.
We first investigate the relationship between the update of projection vectors and $\mathbbm{g}^t$. Note that $\mathbf{W}$ consists of $r$ projection vectors $\mathbf{w}_k \ (k=1,2,...,r)$. From Eq. (\ref{eq:derivative of L to W}), we can deduce that
\begin{equation}\label{eq:updata wk}
\mathbf{w_k^{t+1}} = \mathbf{w_k^t} + \tau^t(\mathbf{x_i}^t(\mathbf{g_i}_{[k]}^t - \mathbf{h_i}_{[k]}^t) + \mathbf{x_j}^t(\mathbf{g_j}_{[k]}^t - \mathbf{h_j}_{[k]}^t)).
\end{equation}
It can be found that $\mathbf{w_k^{t+1}}= \mathbf{w_k^t}$, when $\mathbf{g_i}_{[k]}^t = \mathbf{h_i}_{[k]}^t$ and $\mathbf{g_j}_{[k]}^t = \mathbf{h_j}_{[k]}^t$;
otherwise, $\mathbf{w_k^t}$ will be updated. So the more $\mathbf{w}_k$ in $\mathbf{W}$ we update, the more corresponding hash bits of all data points we subsequently have to update when applied to real-world system.
This takes severely much time which cannot be ignored for online applications. Hence, we should change hash bits as few as possible; in other words, we aim to update $\mathbf{w}_k$ as few as possible. This means that $p$ should be as small as possible, meanwhile guaranteeing that $\mathbbm{g}^t$ satisfies the constrain $R(\mathbbm{g}^t,s^t) =0$. Based on the above discussion, the minimum of $p$ is computed as $p_0 = \lceil \beta r \rceil - D_h(\mathbf{h_i}^t,\mathbf{h_j}^t)$ by setting $D_h(\mathbf{g_i}^t,\mathbf{g_j}^t) = \lceil \beta r \rceil$, as $p = D_h(\mathbf{g_i}^t,\mathbf{g_j}^t) - D_h(\mathbf{h_i}^t,\mathbf{h_j}^t)$ and $D_h(\mathbf{g_i}^t,\mathbf{g_j}^t) \ge \lceil \beta r \rceil \ge \beta r$.
Then $\mathbbm{g}^t$ is ready by selecting $p_0$ hash bits whose indexes are in $\mathcal{K}_0$.

After determining the number of hash bits to update, namely $p_0$, the problem now becomes which $p_0$ bits should be picked up from $\mathcal{K}_0$.
\ws{
To establish the rule, it is necessary to measure the potential loss for every bit of $\mathbf{h}_i^t$ and $\mathbf{h}_j^t$. For this purpose, the prediction loss function in Eq. (\ref{eq:prediction-based loss}) can be reformed as 
\begin{displaymath}	
\begin{split}
 \sum_{\mathbf{h_i}^t_{[k]} \neq \mathbf{g_i}^t_{[k]}} 2\mathbf{h_i}^t_{[k]}\mathbf{w_k^t}^T \mathbf{x_i}^t + 
 \sum_{\mathbf{h_j}^t_{[k]} \neq \mathbf{g_j}^t_{[k]}} 2\mathbf{h_j}^t_{[k]} \mathbf{w_k^t}^T \mathbf{x_j}^t + \sqrt{R(\mathbbm{h}^t,s^t)}.
\end{split}
\end{displaymath}	
This tells that $\mathbf{h_i}^t_{[k]} \mathbf{w_k^t}^T \mathbf{x_i}^t$ or $\mathbf{h_j}^t_{[k]}\mathbf{w_k^t}^T \mathbf{x_j}^t$ are parts of the prediction loss, and thus we use it to measure the potential loss for every bit. 
The problem is which bit should be picked up to optimize. For our one-pass online learning, a large update does not mean a good performance will be gained since every time we update the model only based on a new arrived pair of samples, and thus a large change on the hash function would not suit the passed data samples very well. This also conforms to the spirit of passive-aggressive idea that the change of an online model should be smooth. To this end, we take a conservative strategy by selecting the $p_0$ bits that corresponding to smallest potential loss as introduced below. 
First, the potential loss of every bit w.r.t {$H(\mathbf{W}^t)$} is calculated by
\begin{equation}\label{eq:delta}
\delta_k = min\{\mathbf{h_i}^t_{[k]}\mathbf{w_k^t}^T \mathbf{x_i}^t, \ \ \mathbf{h_j}^t_{[k]} \mathbf{w_k^t}^T\mathbf{x_j}^t\}, \quad \ k \in \mathcal{K}_0.
\end{equation}
We only select the smaller one between $\mathbf{h_i}^t_{[k]}\mathbf{w_k^t}^T \mathbf{x_i}^t$ and $\mathbf{h_j}^t_{[k]} \mathbf{w_k^t}^T\mathbf{x_j}^t$ because we will never set $\mathbbm{g}^t$ simultaneously by $\mathbf{g_i}^t_{[k]} = -\mathbf{h_i}^t_{[k]}$ and $\mathbf{g_j}^t_{[k]} = -\mathbf{h_j}^t_{[k]}$ for any $k \in \mathcal{K}_0$.
After sorting $\delta_k$, the $p_0$ smallest $\delta_k$ are picked up and their corresponding hash bits are updated by the following rule:
\begin{equation}\label{rule_bit_change}
\left\{ \begin{array}{ll}
&\mathbf{g_i}_{[k]} = -\mathbf{h_i}_{[k]}, \ if \ \mathbf{h_i}^t_{[k]}\mathbf{w_k^t}^T \mathbf{x_i}^t \le \mathbf{h_j}^t_{[k]} \mathbf{w_k^t}^T\mathbf{x_j}^t, \\
&\mathbf{g_j}_{[k]} = -\mathbf{h_j}_{[k]}, \ otherwise.
\end{array} \right.
\end{equation}
}

The procedure of obtaining $\mathbbm{g}^t$ for a dissimilar pair is summarized in Algorithm \ref{alg:obtain_g}.
\begin{algorithm}[t]
\footnotesize
\caption{Inference of $\mathbbm{g}^t$ for a dissimilar pair}
\label{alg:obtain_g}
\begin{algorithmic}
\STATE  Calculate the Hamming distance $D_h(\mathbf{h_i}^t,\mathbf{h_j}^t)$ between $\mathbf{h_i}^t$ and $\mathbf{h_j}^t$;
\STATE  Calculate $p_0 = \lceil \beta r \rceil - D_h(\mathbf{h_i}^t,\mathbf{h_j}^t)$;
\STATE  Compute $\delta_k$ for $k \in \mathcal{K}_0$ by Eq. (\ref{eq:delta});
\STATE  Sort $\delta_k$;
\STATE  Set the corresponding hash bits of the $p_0$ smallest $\delta_k$ {opposite to} the corresponding ones in $\mathbbm{h}^t$
by following the rule in Eq.(\ref{rule_bit_change})
and keep the others in $\mathbbm{h}^t$ without change.
\end{algorithmic}
\end{algorithm}

\vspace{0.1cm}

\noindent \textbf{Similar Case}. Regarding similar pairs, the Hamming distance of the optimal hash code pairs $\mathbbm{g}^t$ should be equal or smaller than $\alpha$. Since the Hamming distance between the predicted hash codes of similar pairs may be larger than $\alpha$, we should pick up $p_0$ bits from set $\mathcal{K}_1$ instead of from set $\mathcal{K}_0$, and set them opposite to the corresponding values in $\mathbbm{h}^t$ so as to achieve $R({\mathbbm{g}^t,s^t}) = 0$. Similar to the case for dissimilar pairs as discussed above, the number of hash bits to be updated is $p_0 = D_h(\mathbf{h_i}^t,\mathbf{h_j}^t) - \alpha$, but these bits are selected in $\mathcal{K}_1$. We will compute $\delta_k$ for $k \in \mathcal{K}_1$ and pick up $p_0$ bits with the smallest $\delta_k$ for update.
Since the whole processing is similar to the processing for the dissimilar pairs, we only summarize the processing for similar pairs in Algorithm \ref{alg:obtain_g_sim} and skip the details.

\begin{algorithm}[t]
\footnotesize
\caption{Inference of $\mathbbm{g}^t$ for a similar pair}\label{alg:obtain_g_sim}
\begin{algorithmic}
\STATE  Calculate the Hamming distance $D_h(\mathbf{h_i}^t,\mathbf{h_j}^t)$ between $\mathbf{h_i}^t$ and $\mathbf{h_j}^t$;
\STATE  Calculate $p_0 = D_h(\mathbf{h_i}^t,\mathbf{h_j}^t) - \alpha$;
\STATE  Compute $\delta_k$ for $k \in \mathcal{K}_1$ by Eq. (\ref{eq:delta});
\STATE  Sort $\delta_k$;
\STATE  Set the corresponding hash bits of the $p_0$ smallest $\delta_k$ {opposite to} the corresponding values in $\mathbbm{h}^t$
by following the rule in Eq.(\ref{rule_bit_change})
and keep the others in $\mathbbm{h}^t$ with no change.
\end{algorithmic}
\end{algorithm}

Finally, when $\mathbf{w_k^t}$ is a zero vector, $\delta_k$ is zero as well no matter what the values of $\mathbf{h_i}^t_{[k]}$, $\mathbf{h_j}^t_{[k]}$, $\mathbf{x_i}^t$ and $\mathbf{x_j}^t$ are. This leads to the failure in selecting hash bits to be updated. To avert this, we initialize $\mathbf{W}^1$ by applying LSH. In other words, $\mathbf{W}^1$ is sampled from a zero-mean multivariate Gaussian $\mathcal{N}(0,I)$, and we denote this matrix by $\mathbf{W}_{LSH}$.

\section{Analysis} \label{sec:analysis}

\subsection{Bounds for Similarity Loss and Prediction Loss}\label{sec:Loss_Bound}

In this section, we discuss the loss bounds for the proposed online hashing algorithm. For convenience, at step $t$, we define
\begin{equation}\label{eq:l_u}
\begin{split}
\quad \ell_U^t = \ell^t(\mathbf{U}) = H^t(\mathbf{U}) - G^t(\mathbf{U}) + \sqrt{R(\mathbbm{h}^t,s^t)},
\end{split}
\end{equation}
where $\mathbf{U}$ is an arbitrary matrix in $\mathbb{R}^{d\times r}$. Here, $\ell_U^t$ is considered as the prediction loss based on $\mathbf{U}$ in the $t^{th}$ round.

We first present a lemma that will be utilized to prove Theorem 2.
\begin{lem}\label{lemma:1}
Let $(\mathbbm{x}^1, s^1), \cdots , (\mathbbm{x}^t,s^t)$ be a sequence of pairwise examples, each with a similarity label $s^t \in \{ 1, -1 \}$.
The data pair $\mathbbm{x}^t \in \mathbb{R}^{d \times 2}$ is mapped to a $r$-bit hash code pair $\mathbbm{h}^t \in \mathbb{R}^{r \times 2}$ through the hash projection matrix $\mathbf{W}^t \in \mathbb{R}^{d \times r}$. Let $\mathbf{U}$ be an arbitrary matrix in $\mathbb{R}^{d \times r}$. If $\tau^t$ is defined as that in Eq. (\ref{eq:solution}), we then have
\begin{displaymath}
\begin{split}
\sum_{t=1}^{\infty}\tau^t(2\ell^t(\mathbf{W}^t)- \tau^t ||\mathbbm{x}^t(\mathbbm{g}^t -\mathbbm{h}^t)^T||^2_F - 2\ell_U^t) \le ||\mathbf{U}-\mathbf{W}^1||^2_F,
\end{split}
\end{displaymath}
where $\mathbf{W}^1$ is the initialized hash projection matrix that consists of non-zero vectors.
\end{lem}

\begin{proof}
By using the definition
\begin{displaymath}
\Delta_t = ||\mathbf{W}^t-\mathbf{U}||^2_F-||\mathbf{W}^{t+1}-\mathbf{U}||^2_F,
\end{displaymath}
we can have
\begin{equation}\label{eq:sum_delta}
\begin{split}
\sum_{t=1}^{\infty}\Delta_t = \sum_{t=1}^{\infty}(||\mathbf{W}^t-\mathbf{U}||^2_F-||\mathbf{W}^{t+1}-\mathbf{U}||^2_F)
\\
= ||\mathbf{W}^1-\mathbf{U}||^2_F-||\mathbf{W}^{t+1}-\mathbf{U}||^2_F \le ||\mathbf{W}^1-\mathbf{U}||^2_F.
\end{split}
\end{equation}
From Eq. (\ref{eq:derivative of L to W}), we know $\mathbf{W}^{t+1} = \mathbf{W}^t + \tau^t\mathbbm{x}^t(\mathbbm{g}^t - \mathbbm{h}^t)^T $, so we can rewrite $\Delta_t$ as
\begin{displaymath}
\begin{split}
&\Delta_t = ||\mathbf{W}^t-\mathbf{U}||^2_F-||\mathbf{W}^{t+1}-\mathbf{U}||^2_F
\\
&=||\mathbf{W}^t-\mathbf{U}||^2_F-||\mathbf{W}^t-\mathbf{U} + \tau^t\mathbbm{x}^t(\mathbbm{g}^t - \mathbbm{h}^t)^T||^2_F
\\
&=||\mathbf{W}^t-\mathbf{U}||^2_F - (||\mathbf{W}^t-\mathbf{U}||^2_F  \\
& \ \ \ + 2\tau^t(H^t(\mathbf{W})-G^t(\mathbf{W}) - (H^t(\mathbf{U}) - G^t(\mathbf{U}))) \\
& \ \ \ + (\tau^t)^2||\mathbbm{x}^t(\mathbbm{g}^t - \mathbbm{h}^t)^T||^2_F)
\\
&\ge - 2\tau^t((\sqrt{R(\mathbbm{h}^t,s^t)}-\ell^t(\mathbf{W}^t)) - (\sqrt{R(\mathbbm{h}^t,s^t)}-\ell_U^t)) \\
& \ \ \  - (\tau^t)^2||\mathbbm{x}^t(\mathbbm{g}^t - \mathbbm{h}^T)^T||^2_F
\\
&= \tau^t(2\ell^t(\mathbf{W}^t)- \tau^t ||(\mathbbm{x}^t(\mathbbm{g}^t -\mathbbm{h}^t)^T||^2_F - 2\ell_U^t).
\end{split}
\end{displaymath}
By computing the sum of the left and the right of the above inequality we can obtain
\begin{displaymath}
\sum_{t=1}^{\infty}\Delta_t \ge \sum_{t=1}^{\infty} \tau^t(2\ell^t(\mathbf{W}^t)- \tau^t ||(\mathbbm{x}^t(\mathbbm{g}^t -\mathbbm{h}^t)^T||^2_F - 2\ell_U^t).
\end{displaymath}
Finally, putting Eq. (\ref{eq:sum_delta}) into the above equation, we prove Lemma \ref{lemma:1}.
\end{proof}


\begin{thm}\label{thm:bound_loss}
Let $(\mathbbm{x}^1, s^1), \cdots , (\mathbbm{x}^t,s^t)$ be a sequence of pairwise examples, each with a similarity label $s^t \in \{1, -1\}$ for all $t$. The data pair $\mathbbm{x}^t \in \mathbb{R}^{d \times 2}$ is mapped to a $r$-bit hash code pair $\mathbbm{h}^t \in \mathbb{R}^{r \times 2}$ through the hash projection matrix $\mathbf{W}^t \in \mathbb{R}^{d \times r}$. If $||\mathbbm{x}^t(\mathbbm{g}^t -\mathbbm{h^t})^T||^2_F$ is upper bounded by $F^2$ and \ws{the margin \lk{parameter} $C$ is set as the upper bound of 
$\frac{\sqrt{R(\mathbbm{h}^t,s^t)}}{F^2}$},
then the cumulative similarity loss (Eq. (\ref{eq:loss_function})) is bounded for any matrix $\mathbf{U} \in \mathbb{R}^{d \times r}$, i.e.
\begin{displaymath}
	\sum_{t=1}^{\infty}R(\mathbbm{h}^t,s^t) \le F^2(||\mathbf{U}-\mathbf{W}^1||^2_F + 2\ws{C}\sum_{t=1}^{\infty} \ell_U^t), 
\end{displaymath}
where $C$ is the margin parameter defined in Criterion (\ref{eq:optimization-formulation}).
\end{thm}

\begin{proof}
Based on Lemma \ref{lemma:1}, we can obtain
\begin{equation}\label{ueq:bound}
\begin{split}
\sum_{t=1}^{\infty}\tau^t(2\ell^t(\mathbf{W}^t)- \tau^t ||(\mathbbm{x}^t(\mathbbm{g}^t -\mathbbm{h}^t)^T||^2_F)
\\
 \le ||\mathbf{U}-\mathbf{W}^1||^2_F + 2\sum_{t=1}^{\infty}\tau^t\ell_U^t.
\end{split}
\end{equation}
Based on Eq. (\ref{eq:tau}), we get that
\begin{displaymath}
	\frac{\ell^t(\mathbf{W}^t)}{||\mathbbm{x}^t(\mathbbm{g}^t -\mathbbm{h}^t)^T||^2_F} \ge \tau^t.
\end{displaymath}
\ws{This deduces that
\begin{equation}
\begin{split}
	& \tau^t(2\ell^t(\mathbf{W}^t) - \tau^t ||(\mathbbm{x}^t(\mathbbm{g}^t-\mathbbm{h}^t)^T||^2_F)\\
	\ge & \tau^t(2\ell^t(\mathbf{W}^t) - \ell^t(\mathbf{W}^t)) \ = \  \tau^t\ell^t(\mathbf{W}^t).
\end{split}	
\end{equation}

According to the definition of prediction loss function in Eq. (\ref{eq:prediction-based loss}) and the upper bound assumption
, we know that for any $t$,
\begin{displaymath}
\begin{split}
	& \sqrt{R(\mathbbm{h}^t,s^t)} \le \ell^t(\mathbf{W}^t), \\
	& ||\mathbbm{x}^t(\mathbbm{g}^t -\mathbbm{h}^t)^T||^2_F \le F^2 \text{, and} 
\end{split}	
\end{displaymath}
\begin{equation}\label{eq:lowerbound_c}
	\frac{\sqrt{R(\mathbbm{h}^t,s^t)}}{F^2} \le C
\end{equation}
With these three inequalities and Eq.(\ref{eq:tau}) , it can be deduced that
\begin{equation}\label{ueq:bound_left}
\begin{split}
	\tau^t\ell^t(\mathbf{W}^t) & = 
	min\left\{ \frac{ {\ell^t(\mathbf{W}^t)}^2 }{||\mathbbm{x}^t(\mathbbm{g}^t -\mathbbm{h}^t)^T||^2_F}, C \ell^t(\mathbf{W}^t)  \right\} \\
	& \ge min \left\{ \frac{R(\mathbbm{h}^t,s^t)} {||\mathbbm{x}^t(\mathbbm{g}^t -\mathbbm{h}^t)^T||^2_F}, C\sqrt{R(\mathbbm{h}^t,s^t)}  \right\} \\
	& \ge min \left\{ \frac{R(\mathbbm{h}^t,s^t)}{F^2},  \frac{R(\mathbbm{h}^t,s^t)}{F^2} \right\} \\
	& = \frac{R(\mathbbm{h}^t,s^t)}{F^2}.
\end{split}
\end{equation}
By combining Eq. (\ref{ueq:bound}) and Eq. (\ref{ueq:bound_left}),
we obtain that
\begin{displaymath}
	\sum_{t=1}^{\infty} \frac{R(\mathbbm{h}^t,s^t)}{F^2} \le ||\mathbf{U}-\mathbf{W}^1||^2_F + 2\sum_{t=1}^{\infty} \tau^t \ell_U^t.
\end{displaymath}
Since, $\tau^t \le C$ for all $t$, we have
\begin{equation}\label{eq:thm2}
	\sum_{t=1}^{\infty}R(\mathbbm{h}^t,s^t) \le F^2(||\mathbf{U}-\mathbf{W}^1||^2_F + 2C\sum_{t=1}^{\infty} \ell_U^t). 
\end{equation}
The theorem is proven. }
\end{proof}

\noindent \textbf{Analysis: \footnote{This part is different from the analysis published in in the journal version}}

An optimal projection matrix of our OH model is the one which can predict zero-loss binary codes for any pair of data samples. In other words, an optimal projection matrix can fulfill \lk{zero similarity loss (as defined in Eq. (\ref{eq:loss_function})) at all steps. Since the above theorem is true for arbitrary matrix $\mathbf{U}$, we can find a $\widetilde{\mathbf{U}}$ such that $\sum_{t=1}^{\infty} \ell_{\widetilde{U}}^t = 0$. In this case,} the cumulative similarity loss of the proposed OH is bounded by $F^2||\widetilde{\mathbf{U}}-\mathbf{W}^1||^2_F$, which is a constant and will not grow as $t$ increases.
Based on such an observation, after adequate update, the final hash model, namely the projection matrix $\mathbf{W}$, can converge to an optimal model. 

\lk{The analysis above depends on the $\widetilde{\mathbf{U}}$ such that $\sum_{t=1}^{\infty} \ell_{\widetilde{U}}^t = 0$. The existence of such a matrix is proved as follows. 
From the definition in Eq. (\ref{eq:H}) and Eq. (\ref{eq:G}), we know that $ H^t(\gamma \mathbf{U}) - G^t(\gamma \mathbf{U}) = \gamma \left( H^t(\mathbf{U}) - G^t(\mathbf{U}) \right)$ holds for any constant $\gamma$. If there exists a matrix $\hat{\mathbf{U}}$ such that $\sum_{t=1}^{\infty} \{ H^t(\hat{\mathbf{U}}) - G^t(\hat{\mathbf{U}}) \} \ne 0$, one can always find a constant $\gamma$ such that
\begin{displaymath}
	\sum_{t=1}^{\infty} \{ H^t( \widetilde{\mathbf{U}}) - G^t( \widetilde{\mathbf{U}}) \} = - \sum_{t=1}^{\infty} \sqrt{R(\mathbbm{h}^t,s^t)},
\end{displaymath}
where $\gamma = \frac{-\sum_{t=1}^{\infty} \sqrt{R(\mathbbm{h}^t,s^t)}}{\sum_{t=1}^{\infty} \{ H^t( \hat{\mathbf{U}}) - G^t( \hat{\mathbf{U}}) \}}
	\text{ and }
	\widetilde{\mathbf{U}}= \gamma \hat{\mathbf{U}}.$
Then
\begin{displaymath}
\small
	\sum_{t=1}^{\infty} \ell_{\widetilde{U}}^t =\sum_{t=1}^{\infty} \{ H^t( \widetilde{\mathbf{U}}) - G^t( \widetilde{\mathbf{U}}) \} + \sum_{t=1}^{\infty} \sqrt{R(\mathbbm{h}^t,s^t)} = 0.
\end{displaymath}

Now what is left is to check whether there exists a $\hat{\mathbf{U}}$ having $\sum_{t=1}^{\infty} \{ H^t(\hat{\mathbf{U}}) - G^t(\hat{\mathbf{U}}) \} \ne 0$. If not,  that is $\sum_{t=1}^{\infty} \{ H^t( \mathbf{U}) - G^t( \mathbf{U}) \}  = 0$ for any matrix $\mathbf{U}$. This is a very strong case that makes all the possible projection matrices satisfying an infinite sequence. Mathematically, it means $\lim_{n\rightarrow + \infty}  \sum_{t=1}^{n} \{ H^t( \mathbf{U}) - G^t( \mathbf{U}) \} =0$. It is equivalent that for any $\epsilon > 0$, there exists an integer $N_{\epsilon}$ such that  when $n>N_{\epsilon}$, $|\sum_{t=1}^{n} \{ H^t( \mathbf{U}) - G^t( \mathbf{U}) \}  - 0|<\epsilon$, which holds for any $\mathbf{U}$. However, this is not the case, because if a $\mathbf{U}_0$ satisfies $|\sum_{t=1}^{n} \{ H^t( \mathbf{U}_0) - G^t( \mathbf{U}_0) \}  - 0|=\eta \epsilon$ for some $0<\eta<1$, one can easily find a $\mathbf{U}_0' = \frac{2}{\eta}\mathbf{U}_0$ so that $|\sum_{t=1}^{n} \{ H^t( \mathbf{U}_0') - G^t( \mathbf{U}_0') \}  - 0|=2\epsilon>\epsilon$ which makes the contradiction. 
The only way to avoid this contradiction is $|\sum_{t=1}^{n} \{ H^t( \mathbf{U}) - G^t( \mathbf{U}) \}  - 0|= 0$ for any $n>N_{\epsilon}$ and any $\mathbf{U}$, and this would imply that $H^t( \mathbf{U}) - G^t( \mathbf{U})=0$ holds for any $\mathbf{U}$ when $t>N_{\epsilon}+1$ as well. Note that $H^t( \mathbf{U}) - G^t( \mathbf{U})=0$ also holds for any sample pair $(\mathbf{x_i}^t, \mathbf{x_j}^t)$, that is
$({\mathbf{h_i}^t}^T-{\mathbf{g_i}^t}^T)\mathbf{U}^T\mathbf{x_i}^t+({\mathbf{h_j}^t}^T-{\mathbf{g_j}^t}^T)\mathbf{U}^T\mathbf{x_j}^t=0$. This is only possible when ${\mathbf{h_i}^t}={\mathbf{g_i}^t}$ and ${\mathbf{h_j}^t}={\mathbf{g_j}^t}$; otherwise violation will be reached when scaling $\mathbf{x_i}^t$($\mathbf{x_j}^t$), where the hash code vectors ${\mathbf{h_i}^t}, {\mathbf{g_i}^t}$(${\mathbf{h_j}^t}, {\mathbf{g_j}^t}$) will not be changed after scaling. It means when $t>N_{\epsilon}+1$, the estimated hash code vector $\mathbf{h_i}^t$ is the same as the groundtruth one ${\mathbf{g_i}^t}$ no matter using which projection matrix $\mathbf{U}$ and no matter processing which pair of data points $(\mathbf{x_i}^t, \mathbf{x_j}^t)$. 
This happens only when the data points are much too easy to separate so that Eq. (2) can infer the hash code vector very well no matter using which $\mathbf{U}$, and this would also imply that online learning model would never be triggered because the estimation of hash code is always right from the beginning and thus no update is needed. Hence such the case is rare and almost impossible in practice. In a word, in practice, there always exists some $\hat{\mathbf{U}}$ such that $\sum_{t=1}^{\infty} \{ H^t(\hat{\mathbf{U}}) - G^t(\hat{\mathbf{U}}) \} \ne 0$; otherwise the online hash model will never be triggered for update.

\vspace{0.3cm}

}

\subsection{Time and Space Complexity}\label{sec:time_complexity}

Based on the algorithm summarized in Algorithm \ref{alg:OLhash}, we can find that the time of computing the prediction code for $\mathbbm{x}^t$ is $O(dr)$ and that of obtaining the similarity loss is
$O(r)$. The process of obtaining the zero-loss code pair $\mathbbm{g}^t$ takes at most $O(rlogr+rd)$ with $O(rd)$ to compute all $\delta_k$ and $O(rlogr)$ to sort hash bits according to $\delta_k$. As for the update process of the projection matrix, it takes $O(rd)$. Therefore, the time complexity for training OH at each round is $O(dr+rlogr)$. Overall, if $n$ pairs of data points participate in the training stage, the whole time complexity is $O((dr+rlogr)n)$.
For the space complexity, it is $O(d+dr) = O(dr)$, with $O(d)$ to store the data pairs and $O(dr)$ to store the projection matrix. Overall, the space complexity remains unchanged during training and is independent of the number of training samples.


\section{Multi-Model Online Hashing}\label{sec:MMOH}


%

\ws{In order to make the online hashing model more robust and less biased by current round update, we extend the proposed online hashing from updating one single model to updating the $T$ models.}
\ws{Suppose that we are going to train $T$ models, which are initialized randomly by LSH. Each model is associated to the optimization of its own similarity loss function in terms of Eq. (\ref{eq:loss_function}), denoted by $R_m(\mathbbm{h}^t_m,s^t)$ $(m = 1,2, \cdots, T)$, where $\mathbbm{h}^t_m$ is the binary code of a new pair $\mathbbm{x}^t$ predicted by the $m^{th}$ model at step $t$. At step $t$, if $\mathbbm{x}^t$ is a similar pair, we only select one of the $T$ models to update. To do that, we compute the similarity loss function for each model $R_m(\mathbbm{h}^t_m,s^t)$, and then we select the model, supposed the $m_0^{th}$ model that obtains the smallest similarity loss, i.e., $m_0=\arg\min_m R_m(\mathbbm{h}^t_m,s^t)$. Note that for a similar pair, it is enough that one of the models has positive output, and thus the selected model is the closest one to suit this similar pair and is more easier to update. If $\mathbbm{x}^t$ is a dissimilar pair, all models will be updated if the corresponding loss is not zero, since we cannot tolerate an wrong prediction for a dissimilar pair. By performing online hashing in this way, we are able to learn diverse models that could fit different data samples locally. The update of each model follows the algorithm presented in Section \ref{sec:Optimization}}.


\ws{To guarantee the rationale of the multi-model online hashing, we also provide the upper bound for the accumulative multi-model similarity loss in the theorem below.
\begin{thm}\label{thm:multi_model_bound_loss}
Let $(\mathbbm{x}^1, s^1), \cdots , (\mathbbm{x}^t,s^t)$ be a sequence of pairwise examples, each with a similarity label $s^t \in \{1, -1\}$ for all $t$. The data pair $\mathbbm{x}^t \in \mathbb{R}^{d \times 2}$ is mapped to a $r$-bit hash code pair $\mathbbm{h}^t \in \mathbb{R}^{r \times 2}$ through the hash projection matrix $\mathbf{W}^t \in \mathbb{R}^{d \times r}$. Suppose $||\mathbbm{x}^t(\mathbbm{g}^t_m -\mathbbm{h^t_m})^T||^2_F$ is upper bounded by $F^2$, and the margin parameter $C$ is set \ws{as} the upper bound of $\frac{\sqrt{R^*_m(\mathbbm{h}^t_m,s^t)}}{F^2}$ for all $m$, where $R^*_m(\mathbbm{h}^t_m,s^t)$ is an auxiliary function defined as:
\begin{equation}\label{eq:multi_model_similar_loss}
\scriptsize
	R^*_m(\mathbbm{h}^t_m,s^t) = 
	\left\{
	\begin{aligned}
		R_m(\mathbbm{h}^t_m,s^t) &,& & \text{if the $m^{th}$ model is selected for update at step $t$}, \\
		0 &,& & otherwise.
	\end{aligned}
	\right.
\end{equation}
Then for any matrix $\mathbf{U} \in \mathbb{R}^{d \times r}$, the cumulative {similarity loss} (Eq. (\ref{eq:loss_function})) is bounded, i.e.,
\begin{displaymath}
	\sum_{t=1}^{\infty} \sum_{m=1}^{T} R^*_m(\mathbbm{h}^t_m,s^t) \le TF^2(||\mathbf{U}-\mathbf{W}^1||^2_F + 2C\sum_{t=1}^{\infty} \ell_U^t), 
\end{displaymath}
where $C$ is the margin parameter defined in Criterion (\ref{eq:optimization-formulation}).
\end{thm}

\begin{proof}

Based on Theorem \ref{thm:bound_loss}, the following inequality holds for $m= 1,2,...T$:
\begin{displaymath}
	\sum_{t=1}^{\infty} R^*_m(\mathbbm{h}^t_m,s^t) \le F^2(||\mathbf{U}-\mathbf{W}^1||^2_F + 2C\sum_{t=1}^{\infty} \ell_U^t).
\end{displaymath}
By summing these multi-model similarity losses of all models, Theorem \ref{thm:multi_model_bound_loss} is proved.
\end{proof}
}

\section{Experiments}\label{sec:Experiment}

In this section, extensive experiments were
conducted to verify the efficiency and effectiveness of the proposed OH models from two aspects: metric distance neighbor search and semantic neighbor search. First, four selected datasets are introduced in Sec. \ref{sec:datasets}. And then, we evaluate the proposed models in Sec. \ref{sec:Experiment_sub_effect}. Finally, we make comparison between the proposed algorithms and several related hashing models in Sec. \ref{sec:comparison}.


\subsection{Datasets}\label{sec:datasets}
The four selected large-scale datasets are: Photo Tourism~\cite{Snavely:dataset_Tour}, 22K LabelMe~\cite{Torralba:dataset_LabelMe}, GIST1M~\cite{Dataset:GIST1M} and
CIFAR-10~\cite{Dataset:CIFAR10}, which are detailed below.

\textbf{Photo Tourism}~\cite{Snavely:dataset_Tour}. It is a large collection of 3D photographs including three subsets, each of which has about 100K patches with $64 \times 64$ grayscale. In the experiment,
we selected one subset consisting of 104K patches taken from Half Dome in Yosemite. We extracted 512-dimensional GIST feature vector for each patch and randomly partitioned the whole dataset into a training set with $98$K patches and a testing set with $6$K patches. The pairwise label $s^t$ is generated based on the matching information. That is, $s^t$ is $1$ if a pair of patches is matched; otherwise $s^t$ is $-1$.

\textbf{22K LabelMe}~\cite{Torralba:dataset_LabelMe}. It contains 22,019 images. In the experiment, each image was represented by 512-dimensional GIST feature vector. We randomly selected $2K$ images from the dataset as the testing set, and set the remaining images as the training set.
To set the similarity label between two data samples, we followed \cite{ Liu:KSH, Wang:SSH}: if either one is within the top 5\% nearest neighbors of the other measured by Euclidean distance, $s^t = 1$ (i.e., they are similar); otherwise $s^t = -1$ (i.e., they are dissimilar).

\textbf{Gist1M}~\cite{Dataset:GIST1M}. It is a popular large-scale dataset to evaluate hash models \cite{Heo:SphericalH,DSH:Jin}. It contains one million unlabeled data with each data represented by a 960-dimensional GIST feature vector. In the experiment,
we randomly picked up 500,000 points for training and the non-overlapped 1,000 points for testing.
Owing to the absence of label information, we utilize pseudo label information by thresholding the top 5\% of the whole dataset as the true neighbors of an instance based on Euclidean distance, so every point has 50,000 neighbors.

\textbf{CIFAR-10 and Tiny Image 80M}~\cite{Dataset:CIFAR10}. CIFAR-10 is a labeled subset of the 80M Tiny Images collection~\cite{Dataset:80TinyImage}. It consists of 10 classes with each class containing $6K$ $32 \times 32$ color images, leading to 60K images in total. In the experiment, every image was represented by 2048-dimensional deep features, and 59K samples were randomly selected to set up the training set with the remained 1K as queries to search through the whole 80M Tiny Image collection.

For measurement, the mean average precision (mAP)~\cite{Turpin:MAP,Wu:SSNH} is used to measure the performance of different algorithms, and mAP is regarded as a better measure than precision and recall when evaluating the quality of results in retrieval~\cite{Turpin:MAP,Wu:SSNH}.
All experiments were independently run on a server with CPU Intel Xeon X5650, 12 GB memory and 64-bit CentOS system.

\subsection{Evaluation of the Proposed Methods} \label{sec:Experiment_sub_effect}

In the proposed models, there are three key parameters, namely $\beta$, $C$ and $T$. In this subsection, we mainly investigate the influence of these parameters. Additionally, we will observe the influence of the RBF kernel function on the proposed models.

As stated in Section \ref{sec_gain_g}, the initial projection matrix $\mathbf{W}^1$ is randomly set based on Gaussian distribution, which is similar to the generation of $\mathbf{W}$ in LSH~\cite{Charikar:LSH}. Thus, such an initialization of $\mathbf{W}^1$ is denoted by $\mathbf{W}^1=\mathbf{W}_{LSH}$.
For the similarity threshold, $\alpha$ is set to $0$, because in most applications, the nearest neighbors of a given sample are looked up within 0 Hamming distance. The dissimilar ratio threshold, $\beta$, is set to $0.4$ on CIFAR-10 and Gist 1M and set to $0.5$ on the other two datasets. Besides, the code length $r$ is set as $r=64$ in this section, and the RBF kernel is a default kernel used in the proposed models in the experiments. When a parameter is being evaluated, the others are set as the default values as shown in Table \ref{tab:default_para}.

\begin{table}[h]\caption{the default values of the parameters in MMOH}
\center
\begin{tabular}{|c|c|c|c|c|c|c|}
\hline
 parameter& $T$ & $r$ & $\mathbf{W}^1$ & $\alpha$ & $\beta$ & $C$    \\
\hline
value & 1 & $48$ & $\mathbf{W}_{LSH}$ & $0$ & $0.4 \sim 0.5$ & $1$ \\
\hline
\end{tabular}
\label{tab:default_para}
\end{table}

\subsubsection{\textbf{Effect of Dissimilarity Ratio Threshold $\beta$}}

\begin{figure}[t]
\begin{center}
{\scriptsize
\subfigure[{\scriptsize  Photo Tourism}] 
{
    \label{fig:Tour_beta}
    \includegraphics[height=0.35 \linewidth]{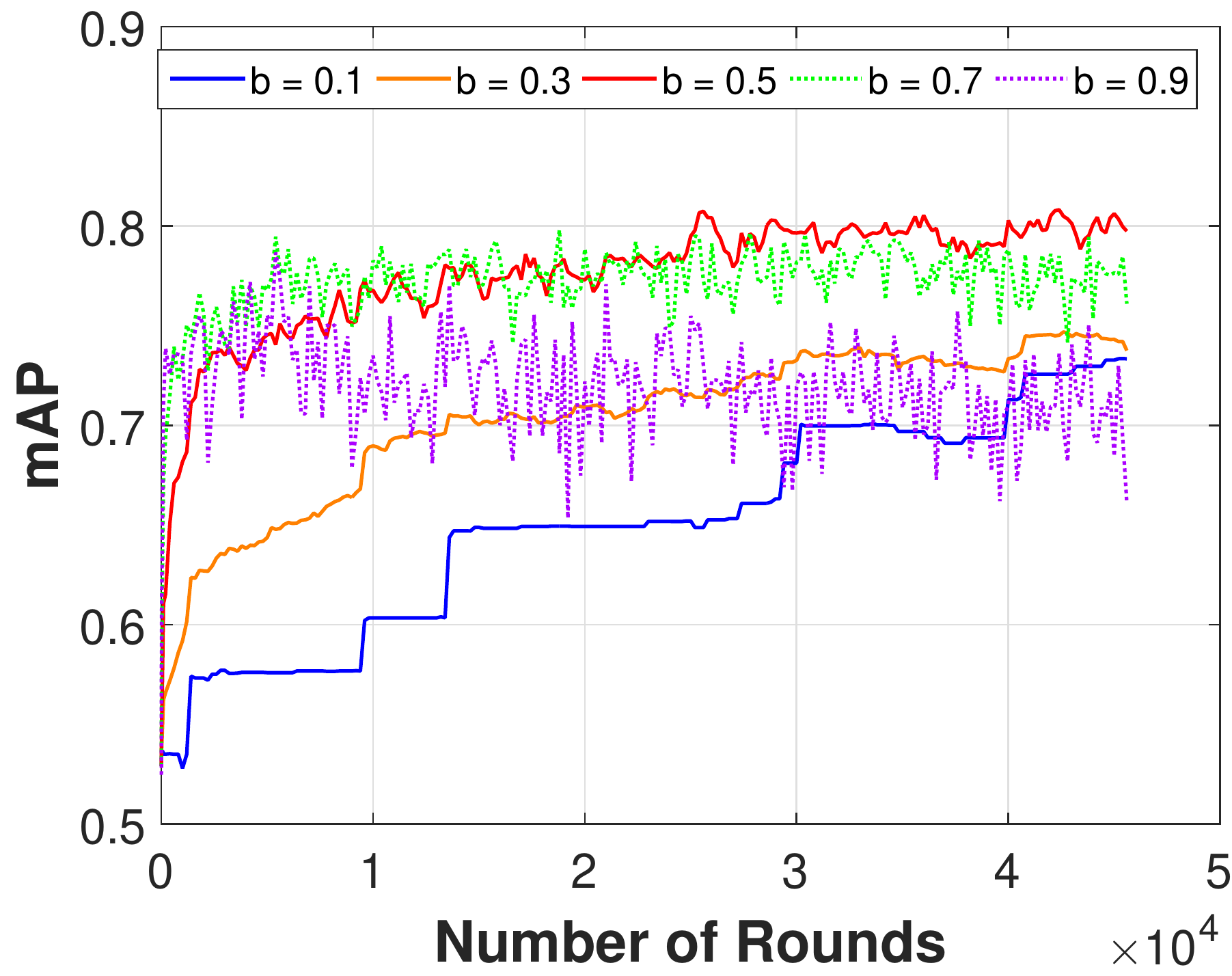}
}
\subfigure[{\scriptsize  22K LabelMe}] 
{
    \label{fig:LM_beta}
    \includegraphics[height=0.35 \linewidth]{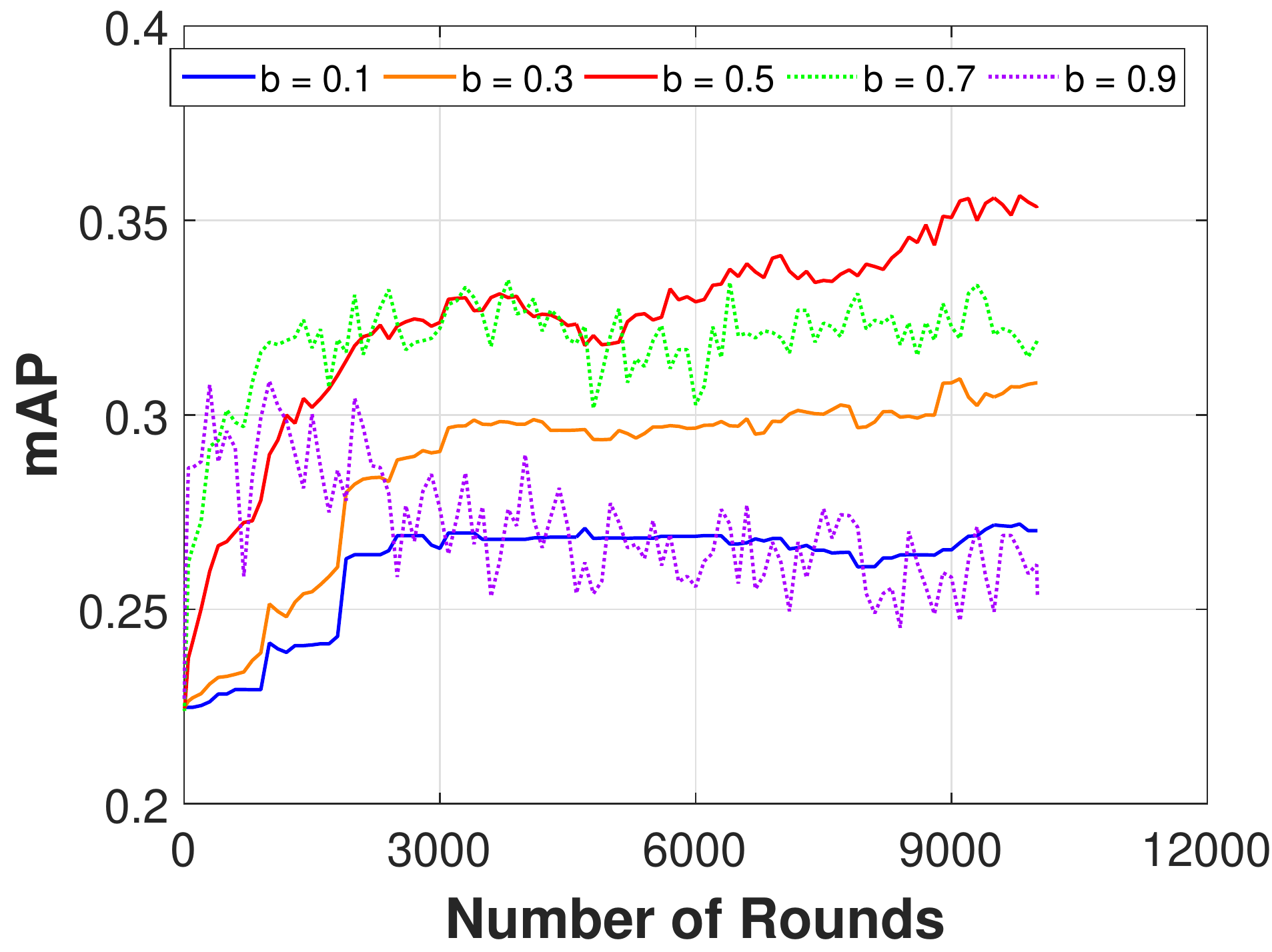}
}
\subfigure[{\scriptsize  GIST1M}] 
{
    \label{fig:gist_beta}
    \includegraphics[height=0.35 \linewidth]{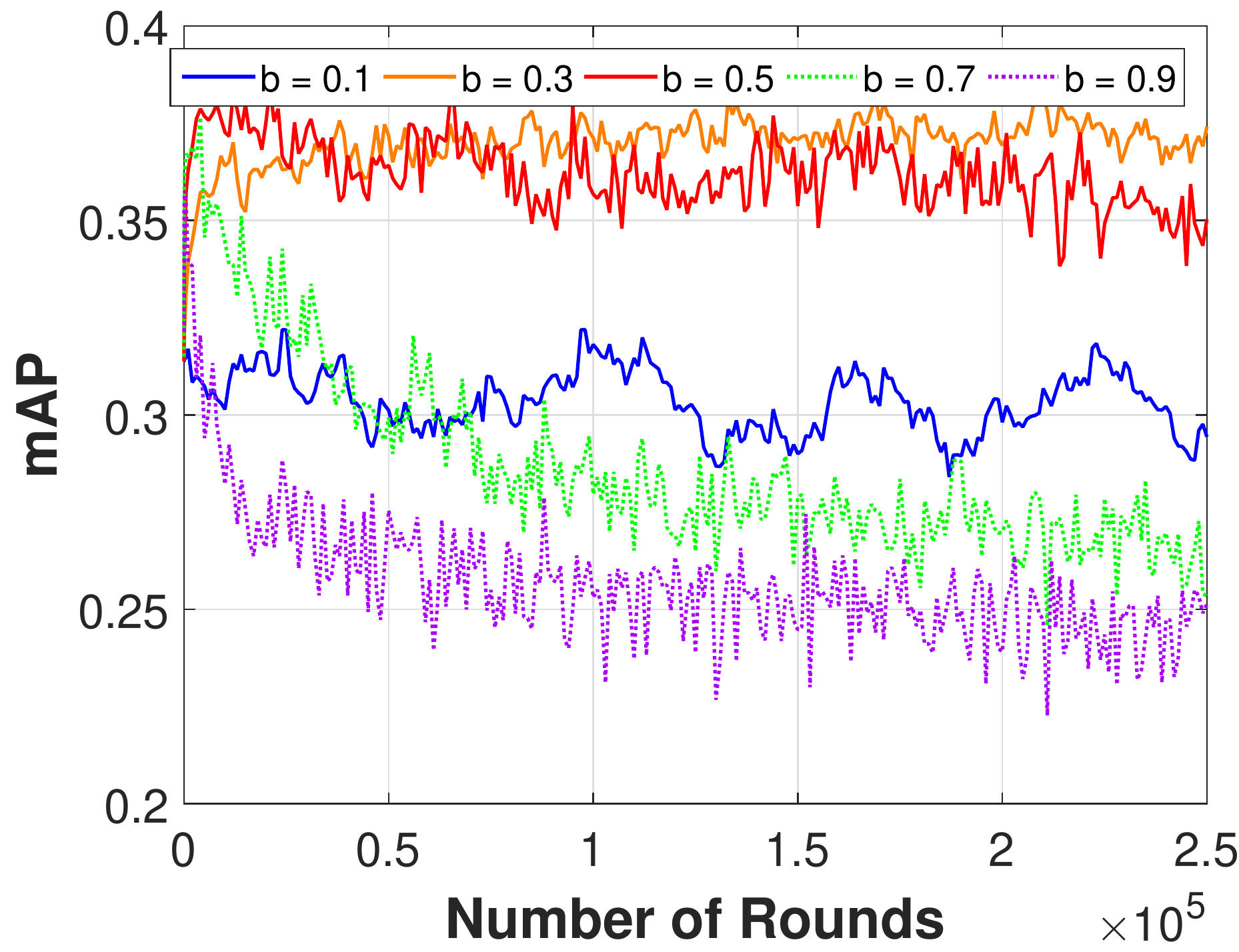}
}
\subfigure[{\scriptsize  CIFAR-10}] 
{
    \label{fig:CIFAR_beta}
    \includegraphics[height=0.35 \linewidth]{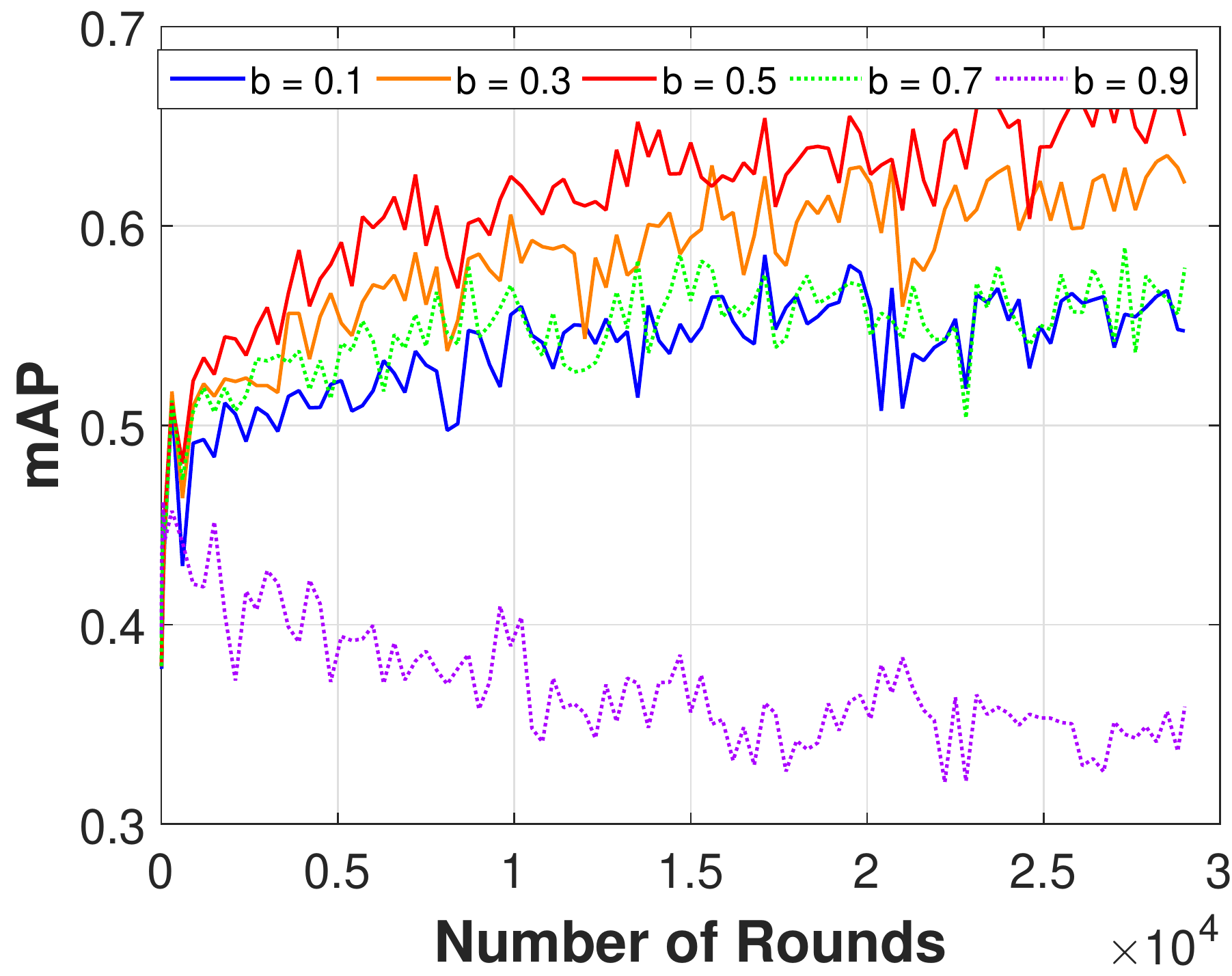}
}
}
\caption{mAP comparison results of OH with respect to different $\beta$ on all datasets. (Best viewed in color.) } 
\label{fig:beta}
\end{center}
\end{figure}

As indicated in Eq. (\ref{eq:loss_function}), $\beta$ is used to control the similarity loss on dissimilar data. A large $\beta$ means that dissimilar data should be critically separated as far as possible in Hamming space. Thus, a too large $\beta$ may lead to excessively frequent update of the proposed OH models. In contrast, a small $\beta$ indicates dissimilar data are less critically separated. Therefore, a too small $\beta$ may make the hash model less discriminative. Consequently, a proper $\beta$ should be set for the proposed online hashing models.

We investigate the influence of $\beta$ on OH on different datasets by varying $\beta$ from $0.1$ to $0.9$. Fig. \ref{fig:beta} presents the experimental results.
On all datasets, when $\beta$ increases, the performance of OH becomes better and better at first. However, when $\beta$ is larger than 0.5, further increasing $\beta$ may lead to performance degradation. 


In summary, a moderately large $\beta$ is better on the other three datasets. 
Based on the experimental results, $\beta=0.5$ is set as a default value for OH on Photo Tourism and 22K LabelMe datasets in the experiments below, and $\beta=0.4$ is the default value for OH tested on GIST 1M and CIFAR-10 datasets.

\ws{
\subsubsection{\textbf{Effect of the Number of Multiple Models $T$}}

\begin{figure}[t]
\centering {\scriptsize
\subfigure[{\scriptsize Photo Touris}] 
{
    \label{fig:Tour_T}
    \includegraphics[height=0.35 \linewidth]{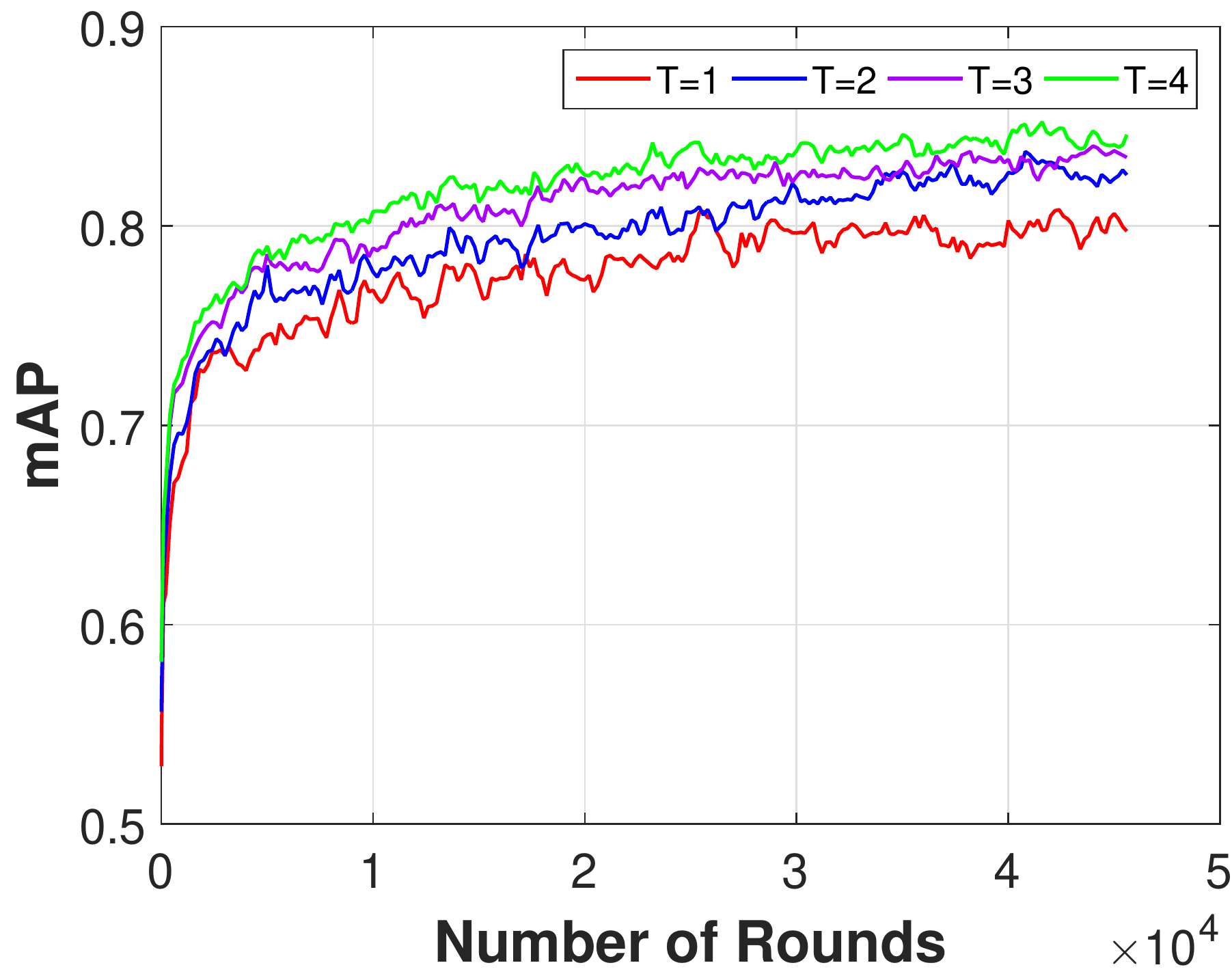}
}
\subfigure[{\scriptsize  22K LabelMe}] 
{
    \label{fig:LM_T}
    \includegraphics[height=0.35 \linewidth]{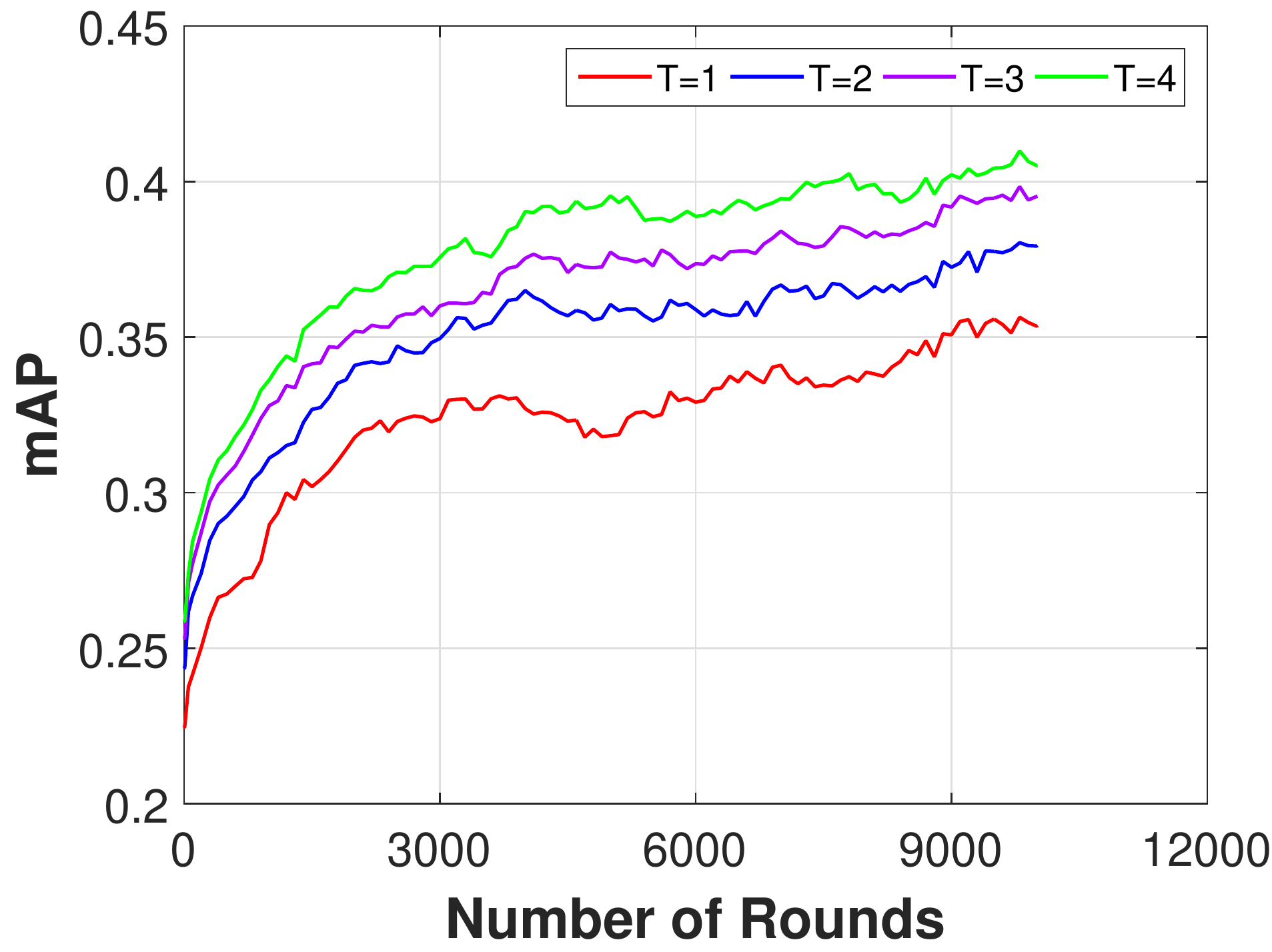}
}
\subfigure[{\scriptsize GIST1M}] 
{
    \label{fig:gist_T}
    \includegraphics[height=0.35 \linewidth]{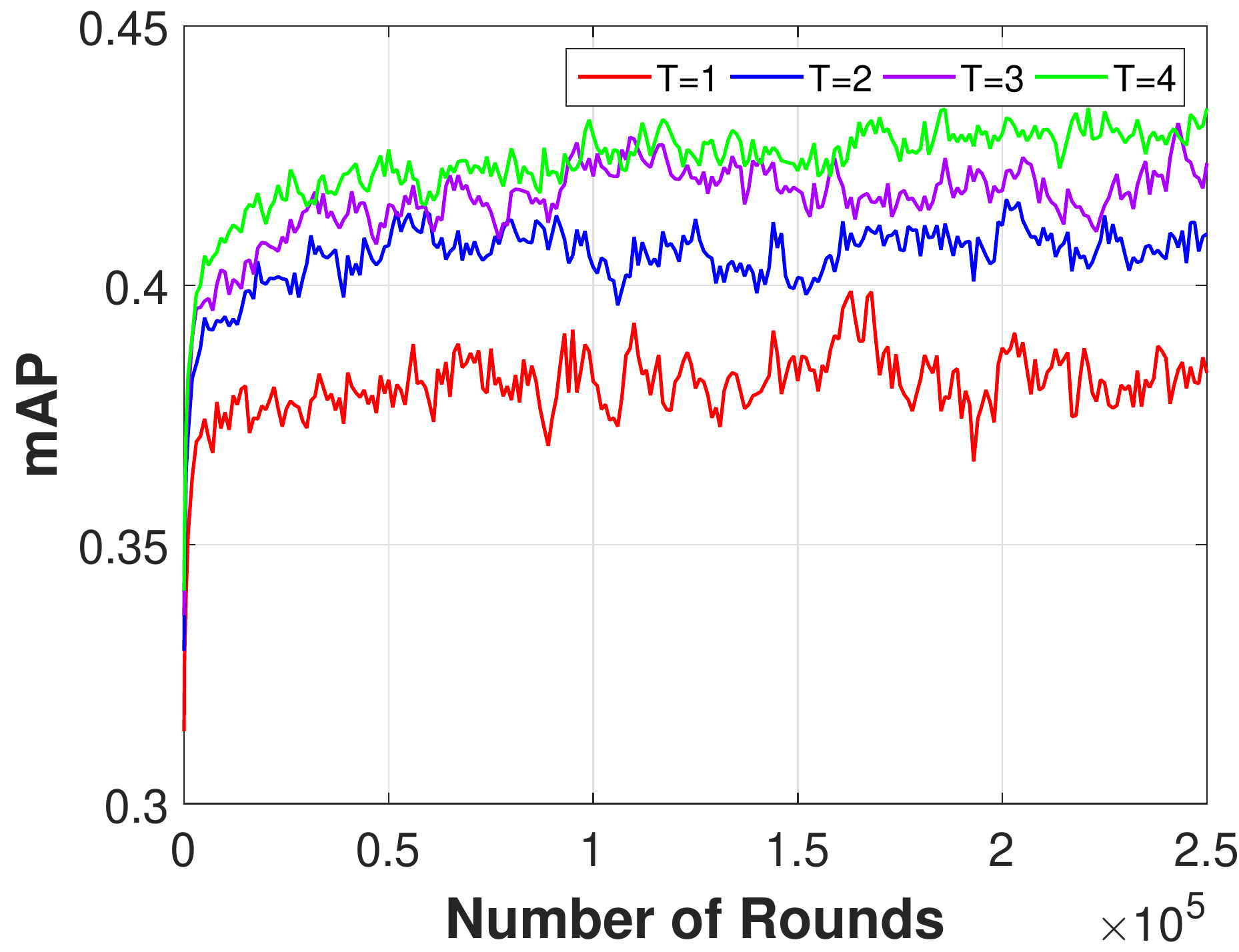}
}
\subfigure[{\scriptsize  CIFAR-10}] 
{
    \label{fig:CIFAR_T}
    \includegraphics[height=0.35 \linewidth]{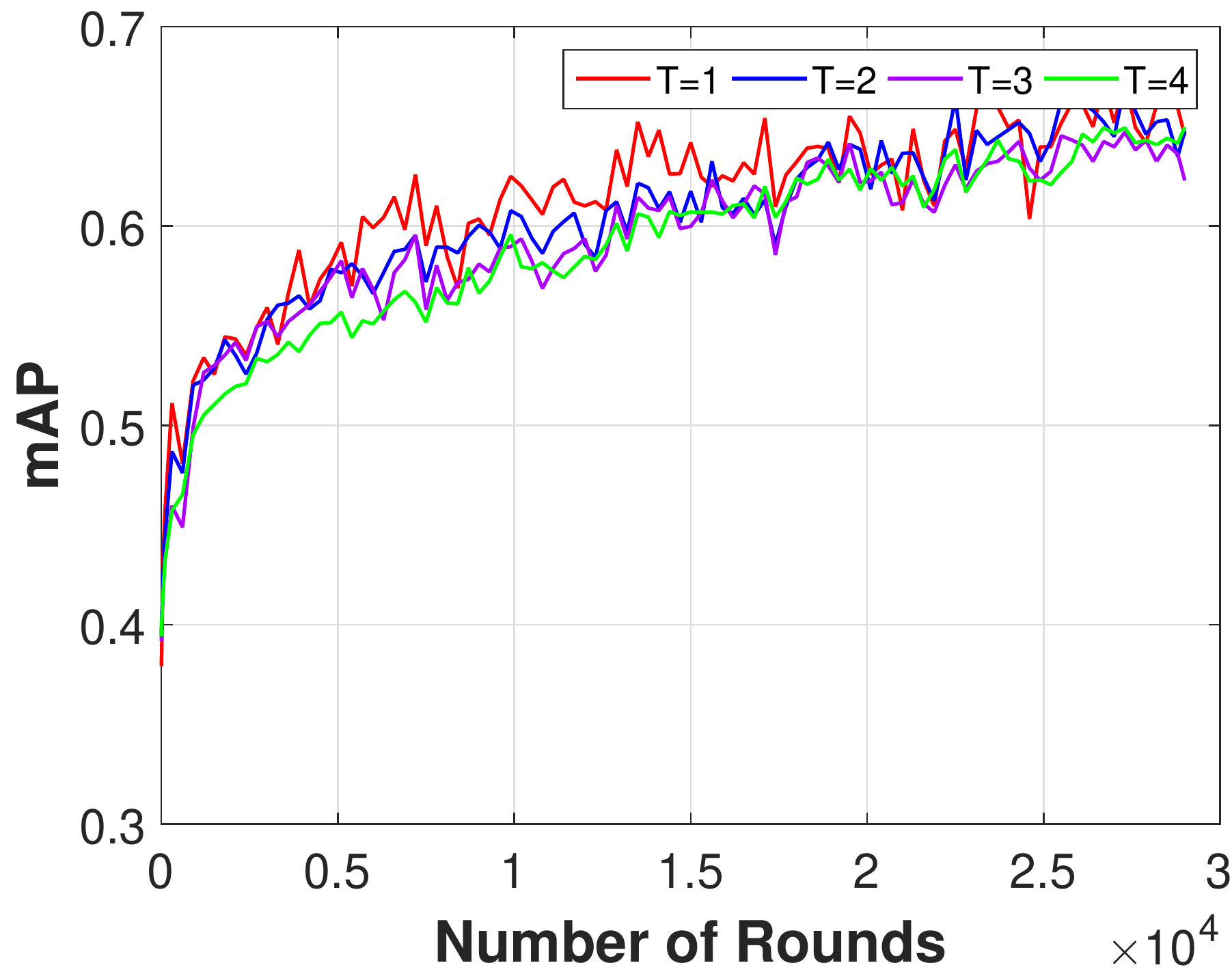}
}
}
\caption{mAP comparison results of MMOH with respect to different $T$ on all datasets. (Best viewed in color.)} 
\label{fig:T}
\end{figure}

Before the investigation about the effect of the number of models $T$ on MMOH, it should be noticed that when $T=1$, MMOH {degrades} to OH. And when $T>1$, we use the multi-index technique \cite{MultiIndexH} to realize fast hash search. 
The influence of $T$ on MMOH is observed by varying $T$ from 1 to 4. Fig. \ref{fig:T} presents the experimental results of MMOH on the four datasets.
{From this figure, we can find that when more models are used (i.e., larger $T$), the performance of MMOH on most datasets except CIFAR-10 is always better. On 22K LabelMe and GIST1M, the improvement of using more models are more clear, and it is less on Photo Tourism. On CIFAR-10, MMOH seems not sensitive to different values of $T$.
In summary, the results in Fig. \ref{fig:T} suggest that MMOH performs overall better and more stably when more models are used.
}

}

\subsubsection{\textbf{Effect of Margin Parameter $C$ }}

\begin{figure}[t]
\centering {\scriptsize
\subfigure[{\scriptsize  Photo Tourism}] 
{
    \label{fig:Tour_c}
    \includegraphics[height=0.35 \linewidth]{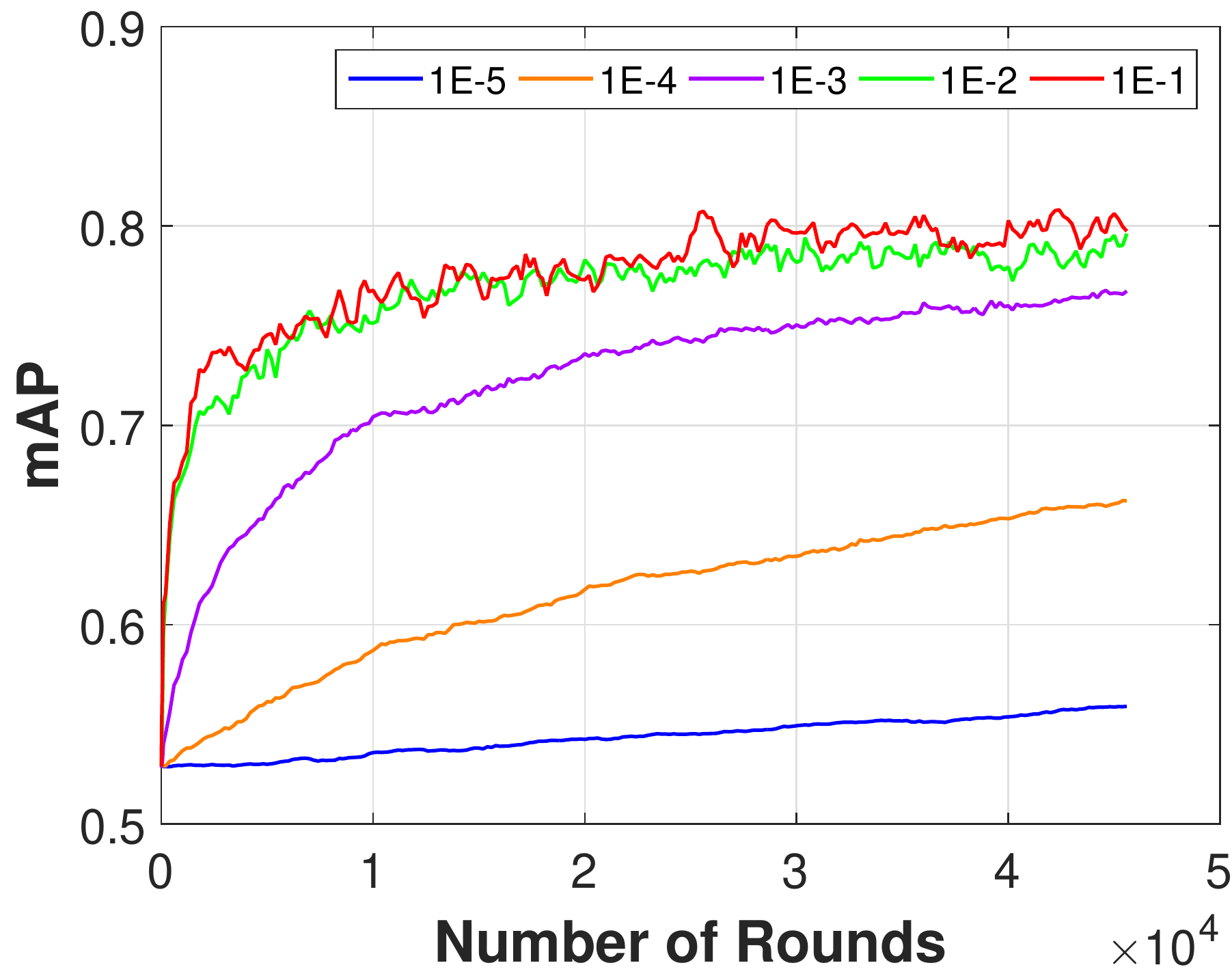}
}
\subfigure[{\scriptsize  22K LabelMe}] 
{
    \label{fig:LM_c}
    \includegraphics[height=0.35 \linewidth]{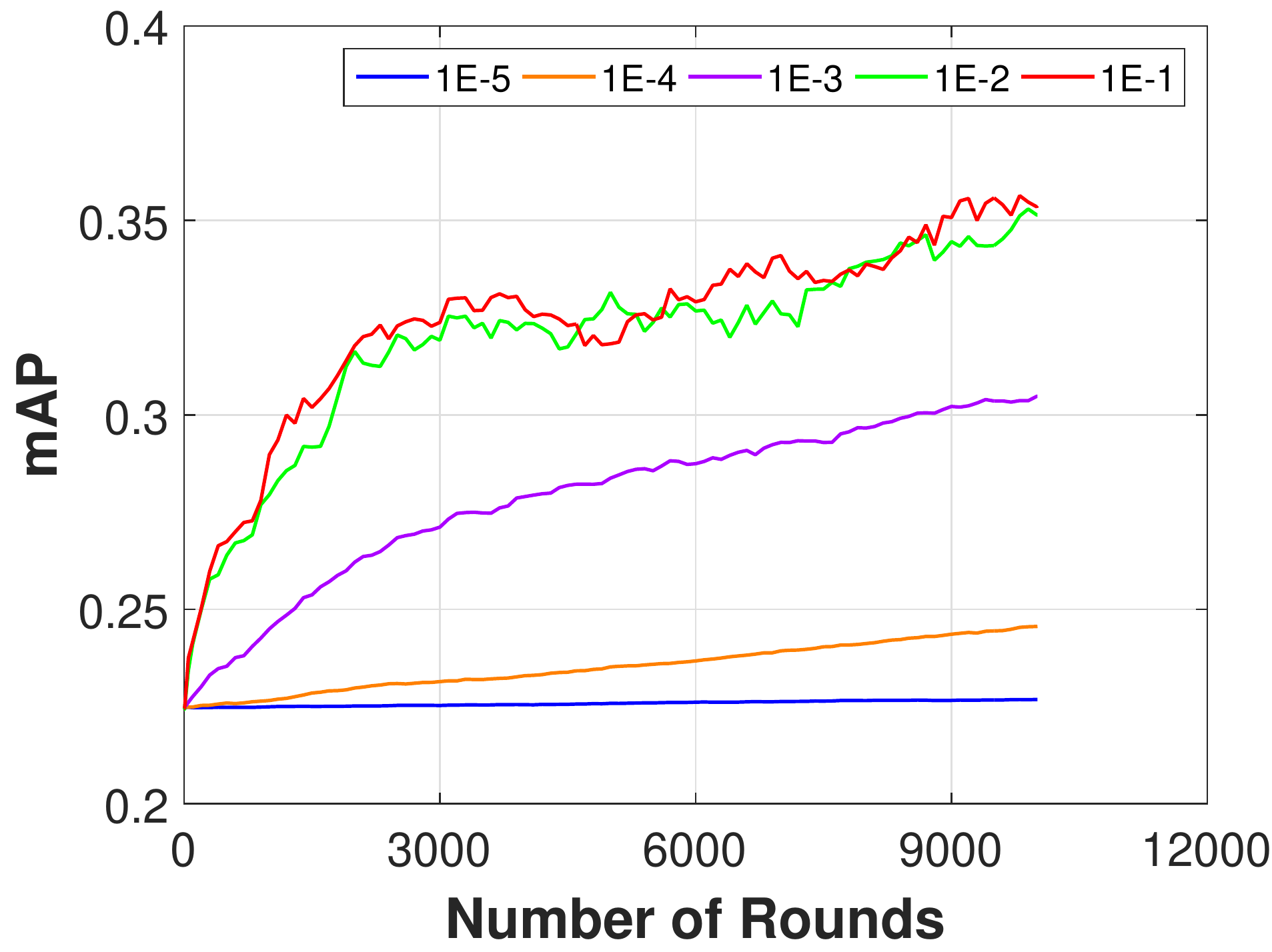}
}
\subfigure[{\scriptsize  GIST1M}] 
{
    \label{fig:gist_c}
    \includegraphics[height=0.35 \linewidth]{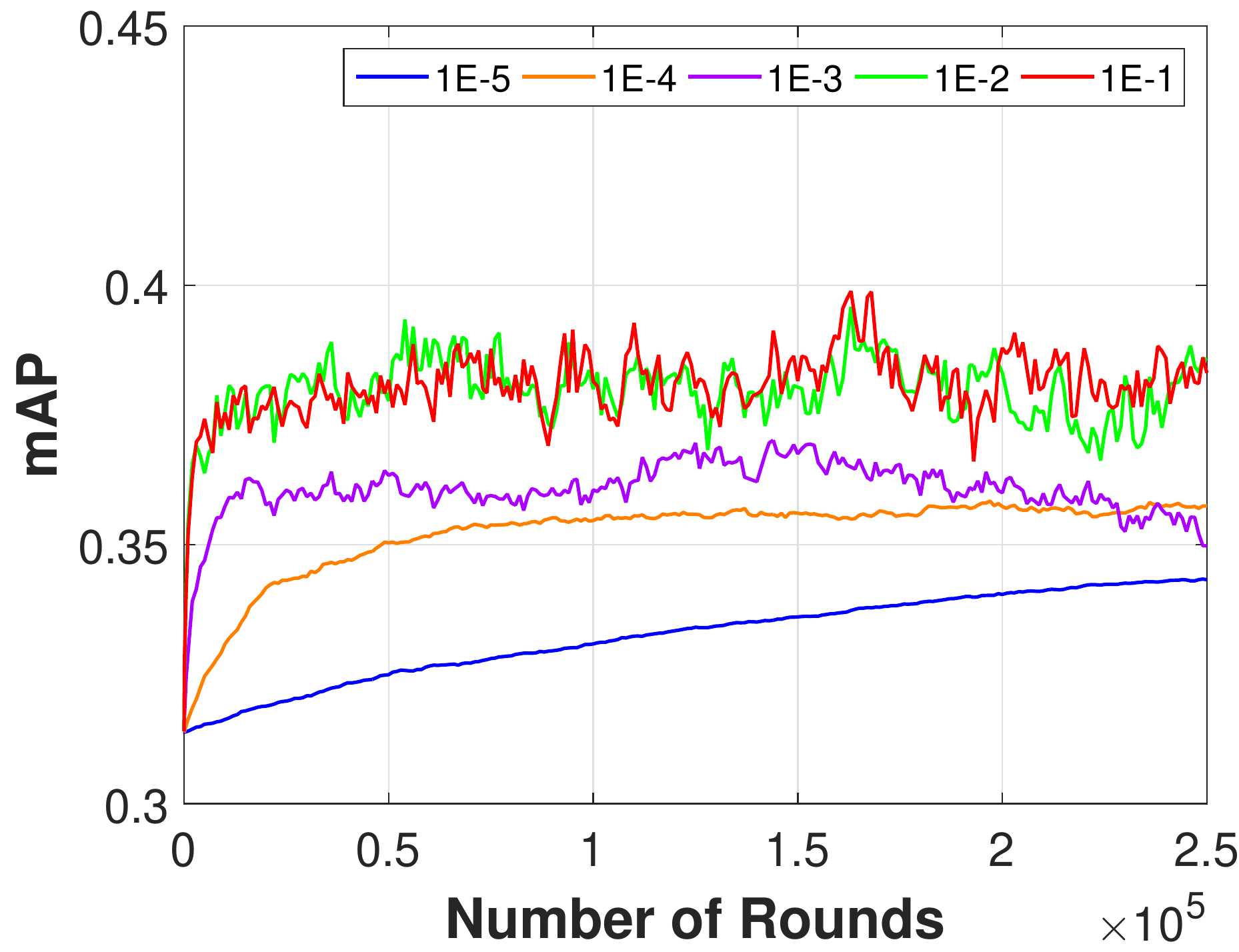}
}
\subfigure[{\scriptsize  CIFAR-10}] 
{
    \label{fig:CIFAR_c}
    \includegraphics[height=0.35 \linewidth]{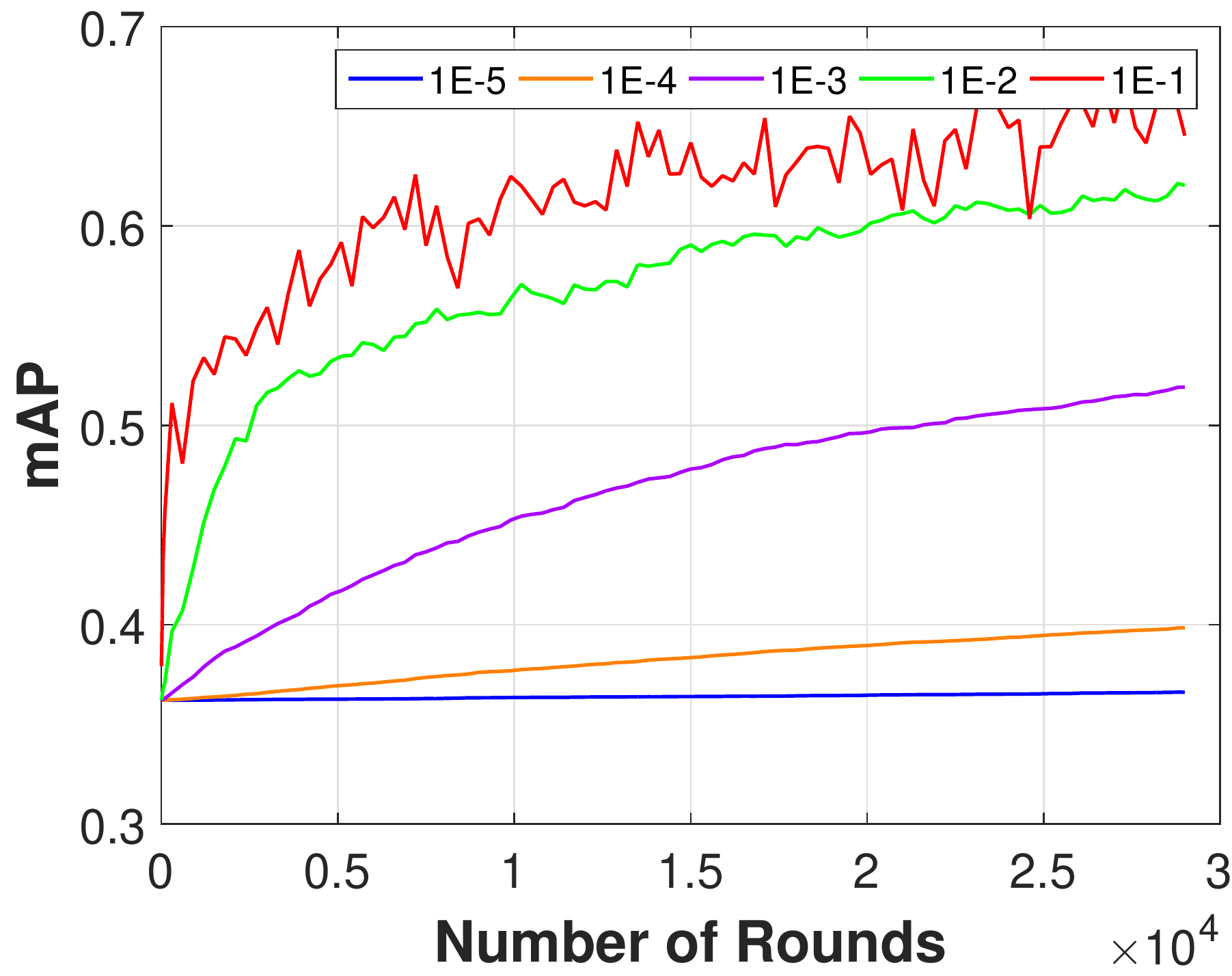}
}
}
\caption{mAP comparison results of OH with $C$ ranging from $0.00001$ to $0.1$. (Best viewed in color.)} 
\label{fig:C}
\end{figure}


In Eq. (\ref{eq:solution}), $C$ is the upper bound of the $\tau^t$, and $\tau^t$ can be viewed as the step size of the update. A large $C$ means that $\mathbf{W}^t$ will be updated more towards reducing the loss on the current data pair.
In Theorem \ref{thm:bound_loss}, a large $C$ is necessary to guarantee the bound on the accumulative loss for OH. \lk{In this experiment, the lower bound of $C$ in Theorem \ref{thm:bound_loss} was $0.017$ since the maximum value $\sqrt{R(\mathbbm{h}^t,s^t)}$ was $8$ and the minimum value of $F^2$ was $479$.}

\lk{The effect of $C$ on the mAP performance was also evaluated on all datasets by varying $C$ from $1E-5$ to $1E-1$.} From Figure \ref{fig:C}, we find that a larger $C$ (i.e. $C=1E-1$) is preferred on all datasets. 
When $C$ increases from $C=1E-5$ to $C=1E-2$, the final performance of OH improves by a large margin on all datasets. 
Further increasing the value of $C$ can only slightly improve the final performance, especially on Photo Tourism, 22K LabelMe and Gist1M. We also find that the performance would nearly not change when setting $C>1E-1$ as this value is larger than $\frac{\ell^t(\mathbf{W}^t)}{||\mathbbm{x}^t(\mathbbm{g}^t - \mathbbm{h}^t)^T||^2_F}$ for most $t$ we have investigated. Hence, $C=1E-1$ is chosen as a default value in the other experiments for OH.

\begin{figure}[t]
\centering {\scriptsize
\subfigure[{\scriptsize  Photo Tourism}] 
{
	\label{fig:cum_loss_pt}
    \includegraphics[height=0.6 \linewidth]{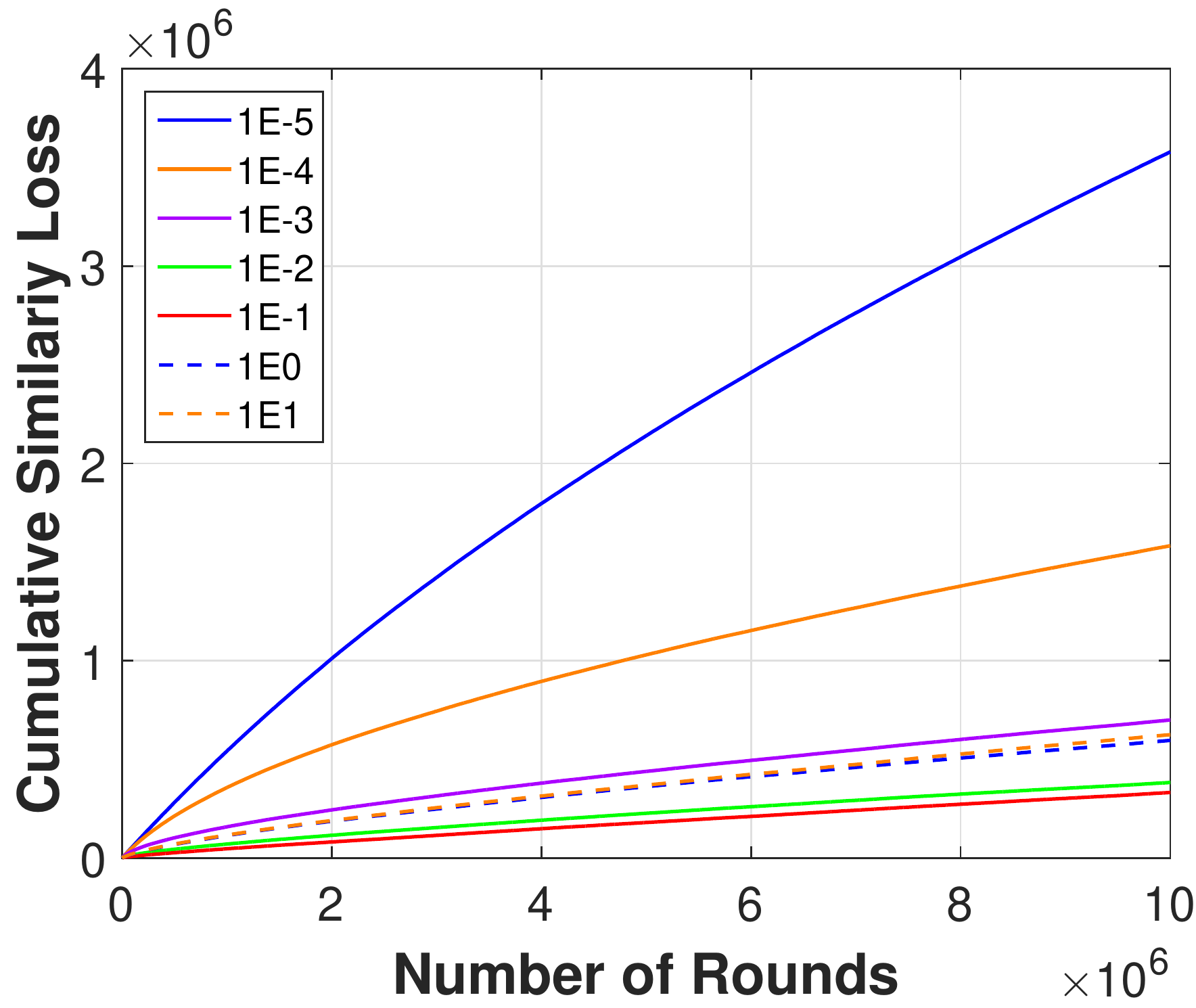}
}

}
\caption{Cumulative similarity loss of OH with $C$ ranging from $0.00001$ to $10$. (Best viewed in color.)} 
\label{fig:C_loss}
\end{figure}

Finally, we evaluated the effect of $C$ on the cumulative similarity loss 
experimentally. For this purpose, we took Photo Tourism dataset as example and ran OH over more than $10^6$ loops by varying $C$ from $1E-5$ to $1E1$, where duplication of training data was allowed. 
The comparison result in Figure \ref{fig:C_loss} shows how the cumulative similarity loss increases when the number of iteration rounds increases. The lowest cumulative loss is achieved by setting $C = 0.1$, which is the one larger but closest to the required lower bound value of $C$ (i.e., $0.017$). When $C$ is too small, i.e. $C < 1E-3$, the cumulative loss grows strongly. This verifies that $C$ should be lower bounded in order to make the cumulative loss under control. 
If $C \ge 1E-3$, the cumulative loss tends to increase much more slowly as shown in Figure \ref{fig:C_loss}. 


\subsection{Comparison with Related Methods}\label{sec:comparison}

To comprehensively demonstrate the efficiency and effectiveness of the proposed OH and MMOH, we further compared it with several related hashing algorithms in terms of mAP and training time complexity.

\subsubsection{ \textbf{Compared Methods}}
There is not much work on developing online hashing methods. To make a proper comparison, the following three kinds of hashing methods were selected:

\begin{itemize}

\item[(a)]
    KLSH~\cite{Kulis:KLSH}, a variant of LSH~\cite{Charikar:LSH} which uses RBF kernel function to randomly sample the hash projection matrix $\mathbf{W}$, is selected as the baseline comparison algorithm. This algorithm is a data-independent hashing method and thus considered as the baseline method.

  \item[(b)]
  
  	Five non-batch based hashing models were selected: LEGO-LSH~\cite{Jain:Fast_Online_Similarity_Search}, MLH~\cite{Norouzi:Min_Loss_Hash}, SSBC~\cite{ArxivStream}, OSH~\cite{OSH} and AdaptHash~\cite{AdaptHash}. LEGO-LSH \cite{Jain:Fast_Online_Similarity_Search}, another variant of LSH, utilizes an online metric learning to learn a metric matrix to process online data, which does not focus on the hash function learning but on metric learning. MLH~\cite{Norouzi:Min_Loss_Hash} and AdaptHash~\cite{AdaptHash} both enable online learning by applying stochastic gradient descent (SGD) to optimize the loss function. SSBC~\cite{ArxivStream} and OSH~\cite{OSH} are specially designed for coping with stream data by applying matrix sketching~\cite{MatrixSketch} on existing batch mode hashing methods. Except AdaptHash and OSH, other three methods all require that the input data should be zero-centered, which is impractical for online learning since one cannot have all data samples observed in advance. 

	\item[(c)] \ws{Four batch mode learning hashing models were selected. One is the unsupervised method ITQ~\cite{Gong:ITQ}. ITQ is considered as a representative method for unsupervised hashing in batch mode.
Since OH is a supervised method, we selected three supervised methods for comparison in order to see how an online approach approximate these offline approaches. The three supervised methods are supervised hashing with kernel (KSH)~\cite{ Liu:KSH}, fast supervised hash (FastHash)~\cite{supervised-hashing20143} and supervised discrete hashing (SDH)~\cite{SupDiscH}. KSH is a famous supervised hashing method to preserve the pairwise similarity between data samples.
FastHash and SDH are two recently developed supervised hashing method in batch mode, where we run SDH for fast search the same as in \cite{ClasswiseHash}. }

\end{itemize}

\subsubsection{ \textbf{Settings}}
Since MLH, LEGO-LSH and OH are based on pairwise data input, we randomly sampled data pairs in sequence from the training set.
For OSH, the chunk size of streaming data was set as 1000 and each data chunk was sequentially picked up from the training sequence as well. Additionally, the key parameters in the compared methods were set as the ones recommended in the corresponding papers. For the proposed OH, the parameters were set according to Table \ref{tab:default_para} and the number of models $T$ is set to be $4$ for MMOH.

\subsubsection{\textbf{Comparison with Related Online Methods}}

\begin{figure}[t]
\centering {\scriptsize
\subfigure[{\scriptsize  Photo Tourism}] 
{
    \label{fig:MAP_Tour}
    \includegraphics[width=0.45 \linewidth]{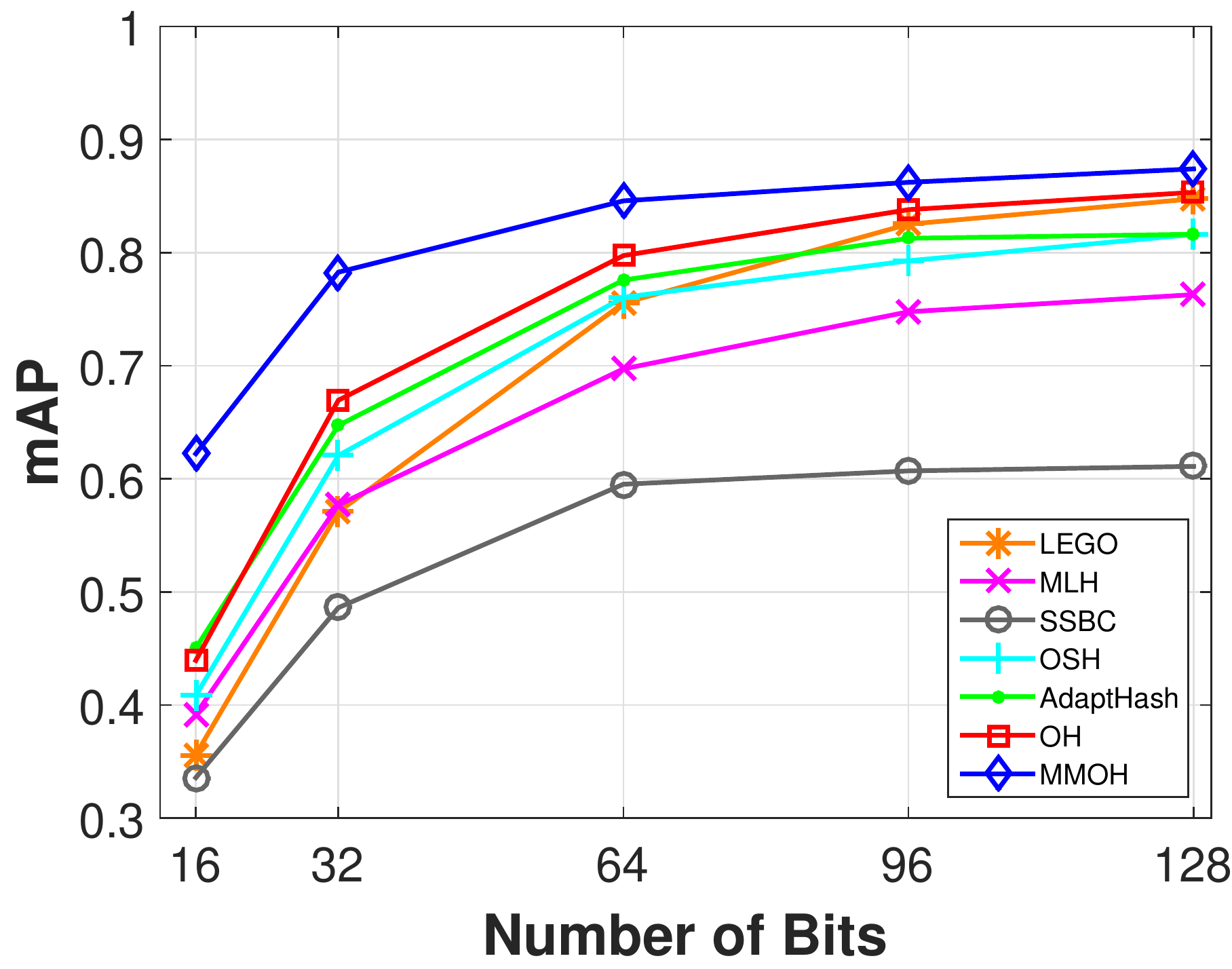}
}
\subfigure[{\scriptsize  22K LabelMe}] 
{
    \label{fig:MAP_LM}
    \includegraphics[width=0.45 \linewidth]{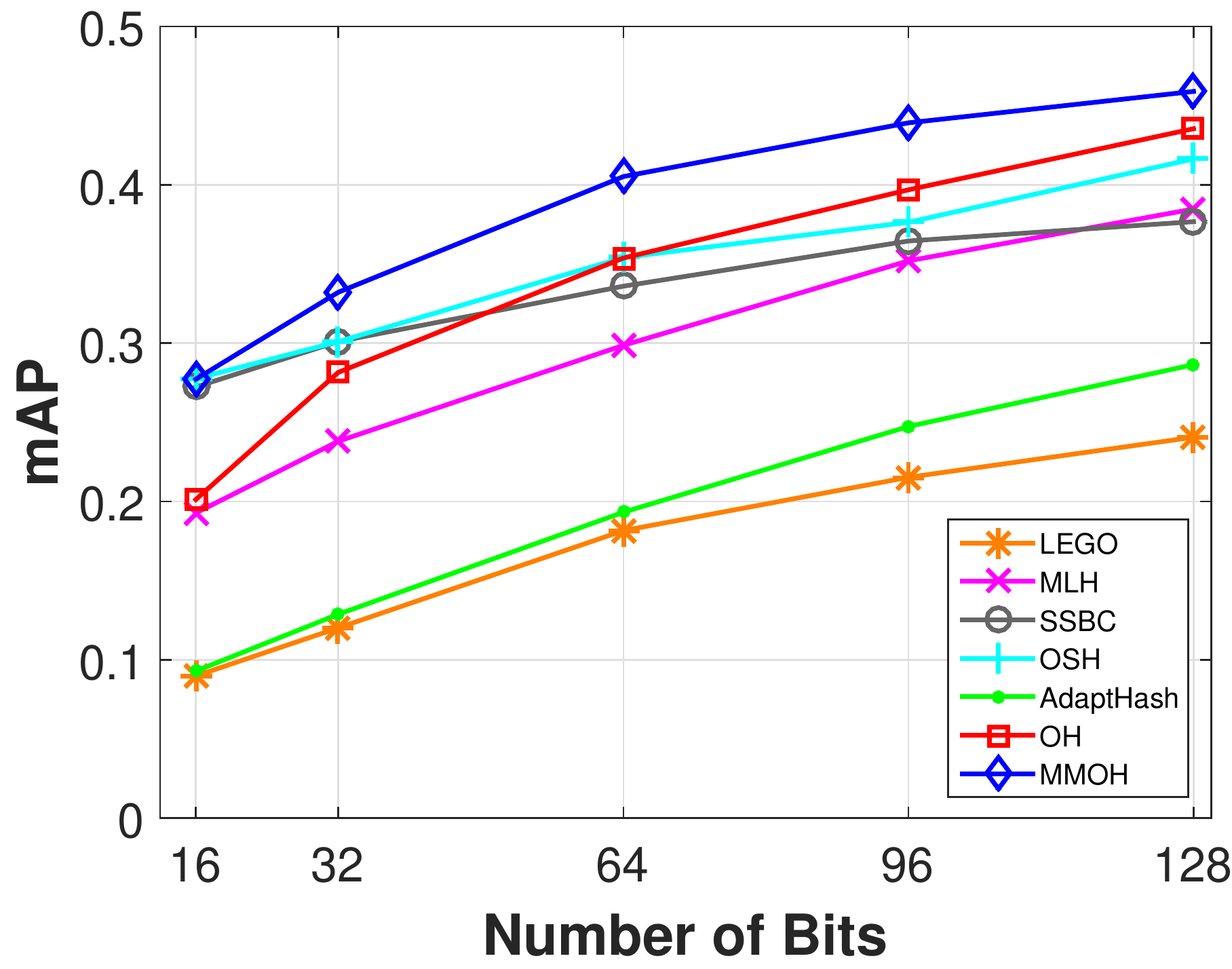}
}
\subfigure[{\scriptsize  GIST 1M}] 
{
    \label{fig:MAP_gist}
    \includegraphics[width=0.45 \linewidth]{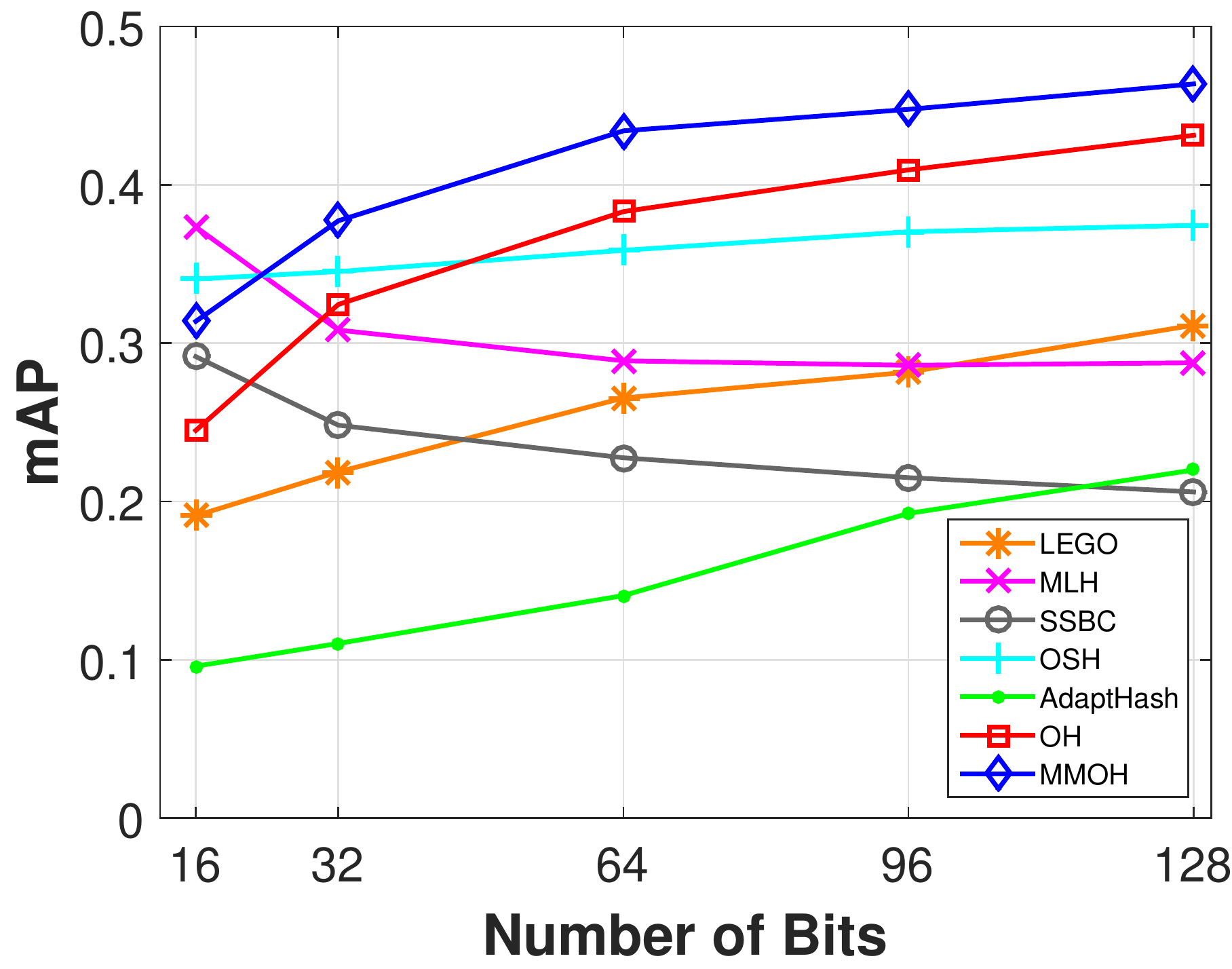}
}
\subfigure[{\scriptsize  CIFAR-10}] 
{
    \label{fig:MAP_CIFAR}
    \includegraphics[width=0.45 \linewidth]{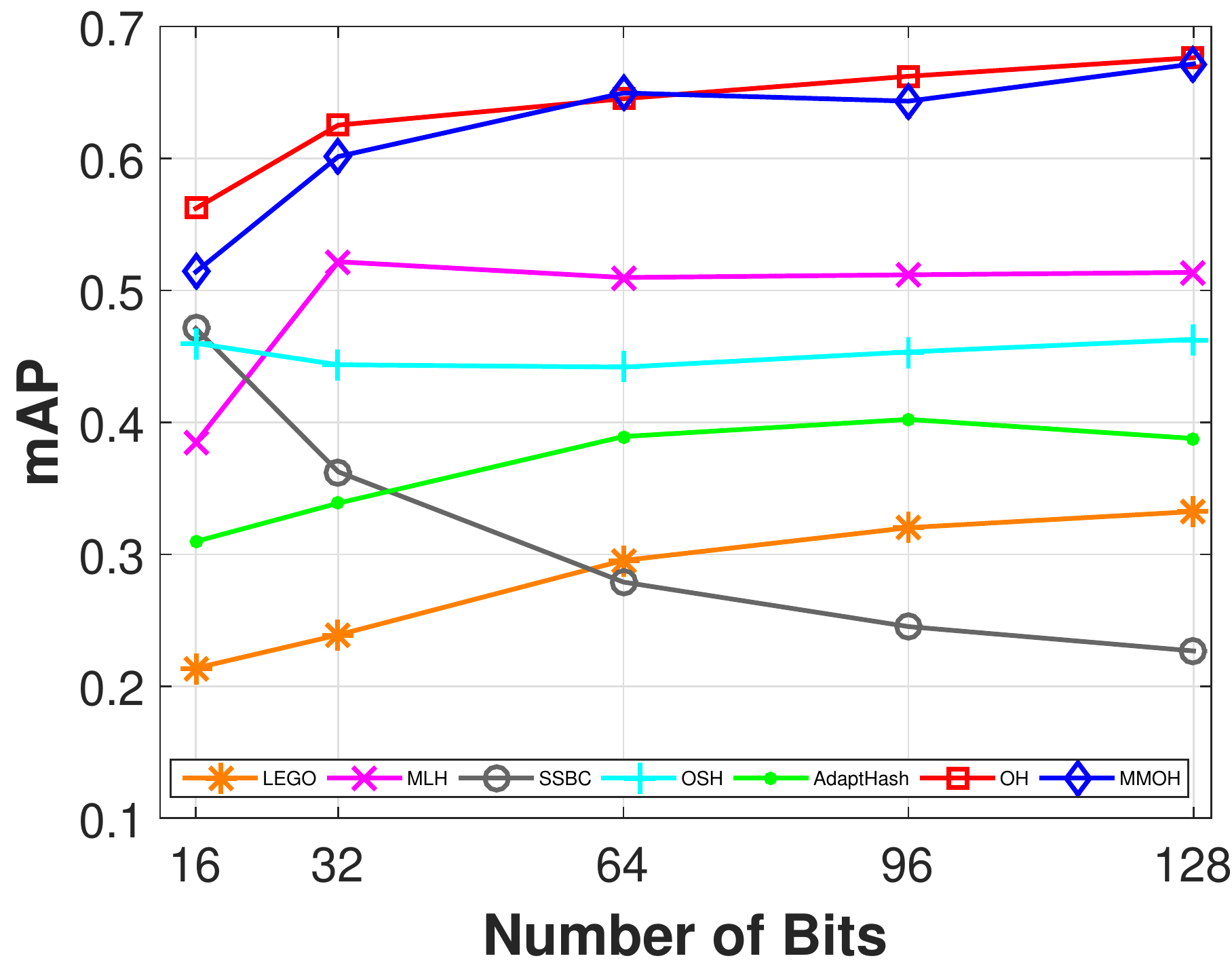}
}
}
\caption{\small mAP comparison results among different online hash models with respect to different code lengths. (Best viewed in color.)} 
\label{fig:MAP}
\end{figure}

We compare OH with the selected related methods in two aspects: mAP comparison and training time comparison.

\vspace{0.2cm}

\noindent {\textbf{- mAP Comparisons}}

\vspace{0.2cm}

Figure \ref{fig:MAP} presents the comparison results among LEGO-LSH, MLH, SSBC, OSH, OH and MMOH with the code length varying from 16 to 128.

From this figure, we find that when the hash code length increases, the performance of MMOH and OH becomes better and better on all datasets. MMOH achieves the best performance and OH achieves the second on almost all datasets when the code length is larger than 32, except Photo Tourism where LEGO-LSH and OH perform very similarly.
Specifically, as the bit length increases, MMOH performs significantly better than OSH and MLH on all datasets, and it performs better than LEGO-LSH on 22K LabelMe, GIST1M and CIFAR-10. In addition, as the hash bit length increases, it is more clear to see MMOH performed better than the compared methods on CIFAR-10.
All these observations demonstrate the efficiency and effectiveness of OH and MMOH to the compared hashing algorithms on processing stream data in an online way.

\vspace{0.2cm}

\noindent \textbf{- Training Time Comparisons}

\vspace{0.2cm}

In this part, we investigate the comparison on the accumulated training time, which is another main concern for online learning. Without loss of generality, the Photo Tourism dataset was selected and 80K training points were randomly sampled to observe the performance of different algorithms. All results are averaged over 10 independent runs conducted on a server with Intel Xeon X5650 CPU, 12 GB memory and 64-bit CentOS system. For fairness, in each run, all experiments were conducted using only one thread. Specially, SSBC is excluded in this experiment since its training time is significantly longer than any other methods. 

First, we investigate the training time of different algorithms when the number of training samples increases. In this experiment, the hash bit length was fixed to 64 and the comparison result is presented in Figure \ref{fig:time_sample}. From this figure, we find that:
 \begin{itemize}
 	\item [1)]As the number of training samples increases, the accumulated training time of all methods almost increases linearly. However, it is evident that LEGO-LSH increases much faster than the other compared methods. In particular, OH, AdaptHash and MMOH increase the lowest\ws{, and in particular OH takes only 0.0015 seconds (using single thread run in a single CPU) for each pair for update.}

 	\item [2)]Compared with LEGO-LSH and OSH, the accumulated training time of other three methods are considerably much smaller. Specifically, when the number of samples is $8\times 10^4$, the accumulated training time of LEGO-LSH is 10 times more as compared to MMOH and MLH. Besides, the time of OSH is 2 times more as compared to MMOH and MLH and is 4 times more as compared to OH and AdaptHash.

 	\item [3)] The training time of MMOH and MLH is very similar. The training time of OH is slightly less than AdaptHash, and it is always the least one.

\end{itemize}

\begin{figure}[t]
\centering {\scriptsize

\subfigure[{\scriptsize Training time comparison among different algorithms when the number of samples increases.}] 
{
    \label{fig:time_sample}
    \includegraphics[width=0.7 \linewidth,height=0.5 \linewidth]{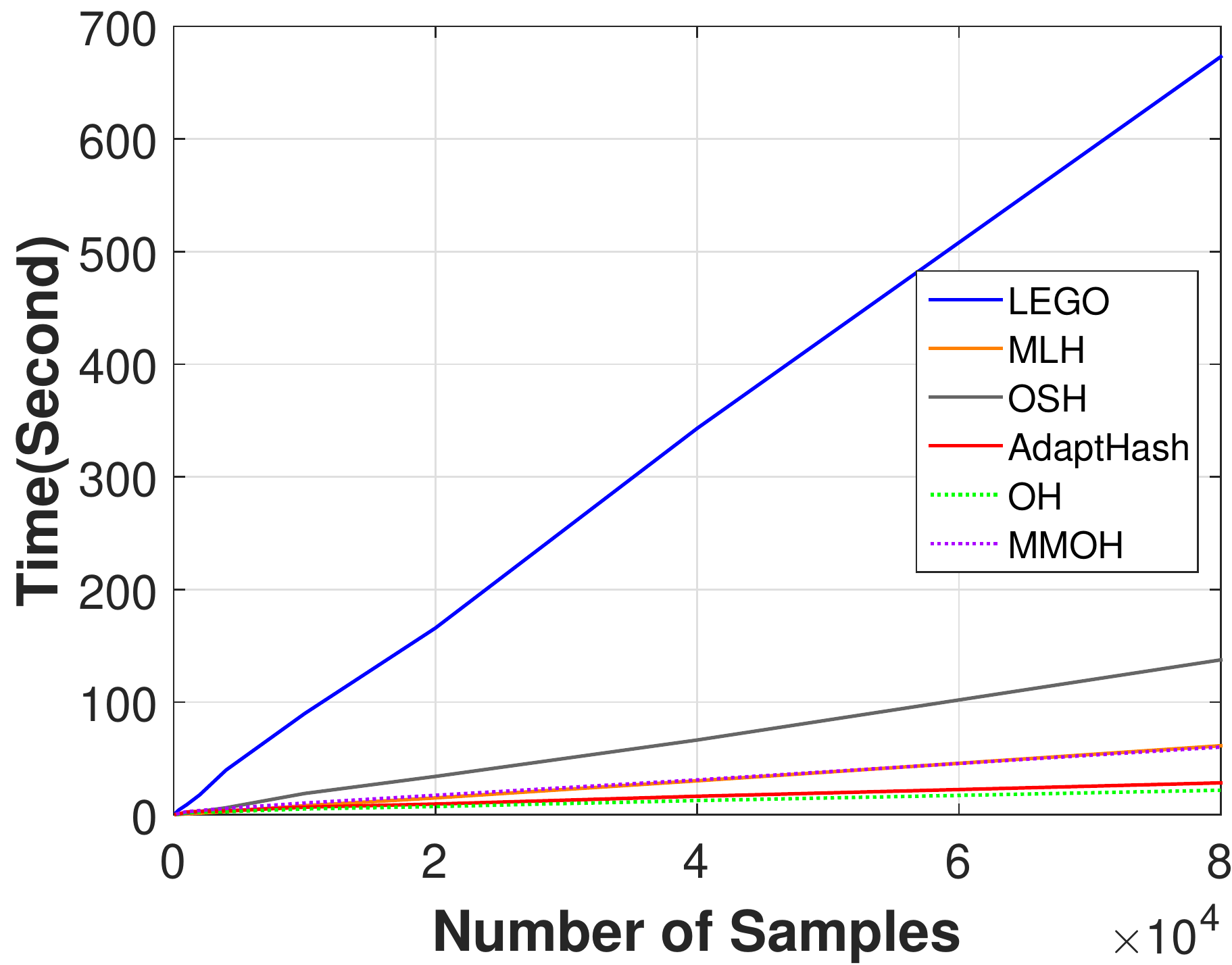}
}
\subfigure[{\scriptsize Training time comparison among different algorithms with different code lengths.}] 
{
    \label{fig:time_r}
    \includegraphics[width=0.7 \linewidth, height=0.5 \linewidth]{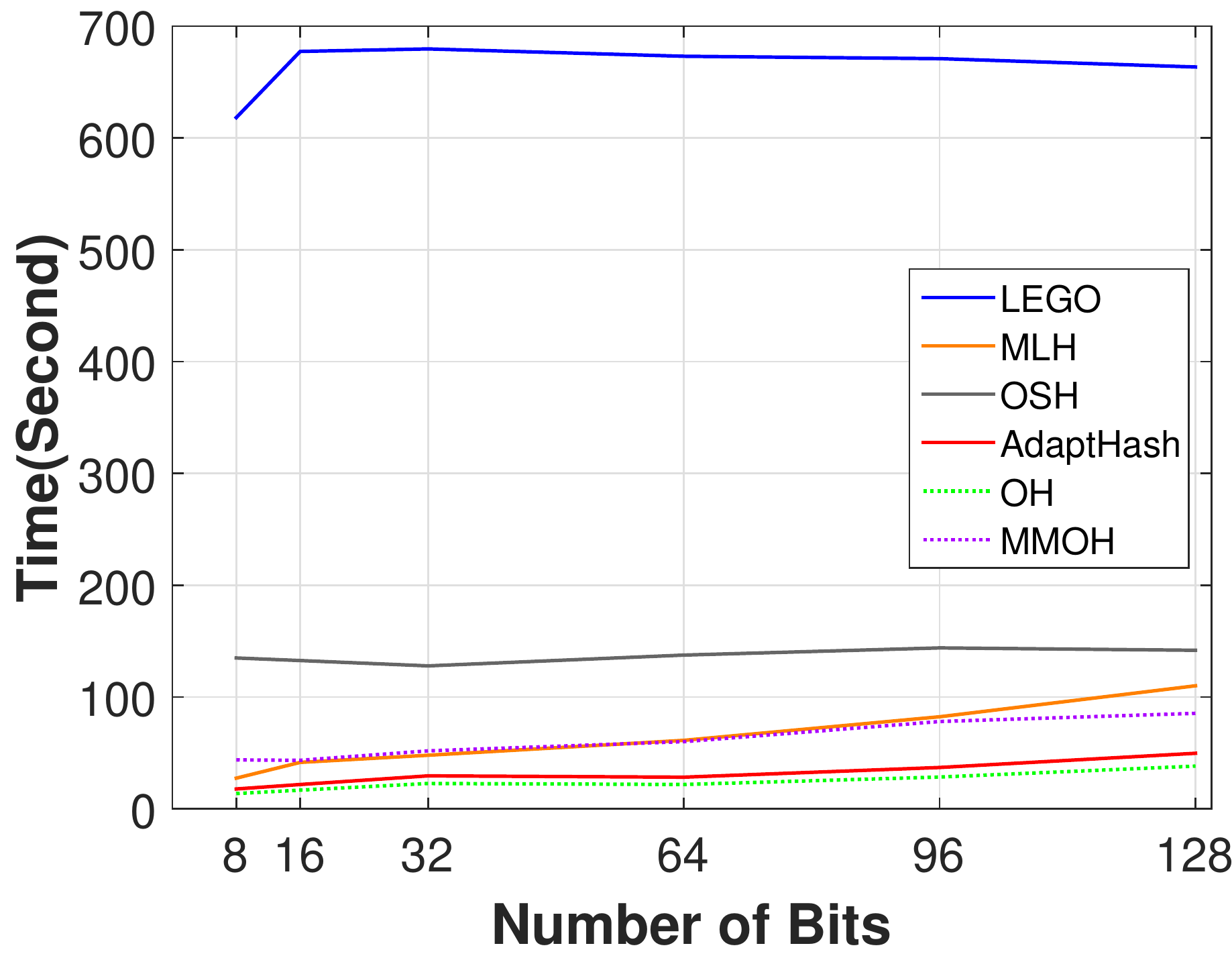}
}
}
\caption{\small Training time comparison among different algorithms on Photo Tourism. (Best viewed in color.)} 
\label{fig:time}
\end{figure}

Second, we further investigate the training time of different methods when the hash code length increases, where the training sample size is fixed to be 80,000. The comparison result is displayed in Figure \ref{fig:time_r}. From this figure, we find that:
\begin{itemize}
 	\item [1)]When the code length increases, the accumulated training time of all algorithms increases slightly except LEGO-LSH. The training time of LEGO-LSH does not change obviously when the code length is larger than 24. This is because most of the training time was spent on the metric training, which is independent of the hash code length, while the generation of hash projection matrix in LSH costs very little time \cite{Jain:Fast_Online_Similarity_Search}.
 	\item [2)]MMOH took considerably smaller accumulated training time than LEGO-LSH and OSH. This is because only a small part of hash bits are updated in MMOH, which reduces the time cost in updating.
    \item [3)] Compared to MLH, MMOH took a little more time than MLH when the bit length is smaller than 64. However, when the bit length is larger than 64, the training time of MMOH becomes slightly less than that of MLH. This phenomenon can be ascribed to the fact that the time spent on selecting the hash model to preserve in MMOH does not heavily depend on the code length, and the time cost of updating hash model in MMOH grows slower than that in MLH.
    \item [4)] OH took the least training time in all the comparison. The training time of AdaptHash is slightly higher. And the training time of OH and AdaptHash is only about half of the training time of MLH and MMOH.
 \end{itemize}

From the above investigation, we can conclude that OH and MMOH are very efficient in consuming considerably small accumulated training time and thus are more suitable for online hash function learning.

\subsubsection{\textbf{Comparison to Batch Mode Methods}}
\ws{
Finally, we compare batch mode (offline) learning-based hashing models. The main objective here is not to show which is better, as online learning and offline learning have different focuses. The comparison here is to show how well an online model can now approach the largely developed batch mode methods.
}

Figure \ref{fig:KLSH-MAP} presents the comparison results on different hash code length varying from 16 to 128 on the four datasets. SDH is only conducted on Photo Tourism and CIFAR-10 as only these two datasets provides class label information.
From this figure, it is evident that OH significantly outperforms KLSH when using different numbers of hash bits on all datasets. Specifically, when the code length increases, the difference between OH and KLSH becomes much bigger, especially on Photo Tourism and 22K LabelMe. 

\ws{ When compared with the unsupervised methods ITQ, OH outperforms it when the code length is larger than 32 bits over all datasets. 
These comparison results demonstrate that OH is better than the compared data-independent methods KLSH and the data-dependent unsupervised methods ITQ overall.
}

\ws{
For comparison with three supervised methods, the mAP of MMOH is the highest on GIST 1M dataset, slightly higher than KSH and FastHash, when the code length is larger than 16 bits. KSH and FastHash have their limitation on this dataset. KSH cannot engage all data in training as the similarity matrix required by KSH is of tremendous space cost, and the Block Graph Cut technique in FastHash may not work well on this dataset as the supervised information is generated by the Euclidean distance between data samples. On the other three datasets, MMOH is not the best, but performs better than two unsupervised methods KLSH and ITQ.
Indeed, from the performance aspect, when all labeled data are available at one time using OH and MMOH is not the best choice as compared to batch mode supervised hashing methods. However, OH and MMOH solve a different problem as compared to the supervised batch ones. OH and MMOH concern how to train supervised hashing model on stream data in an online learning framework, while the batch ones are not. A bound on the cumulative loss function is necessary for an online learning model while it is not necessary for batch mode methods. Hence, OH and MMOH have their unique merits.
%
%
}

\begin{figure}[t]
\centering {\scriptsize
\subfigure[{\scriptsize  Photo Tourism}] 
{
    \label{fig:KLSH_Tour}
    \includegraphics[height=0.35 \linewidth]{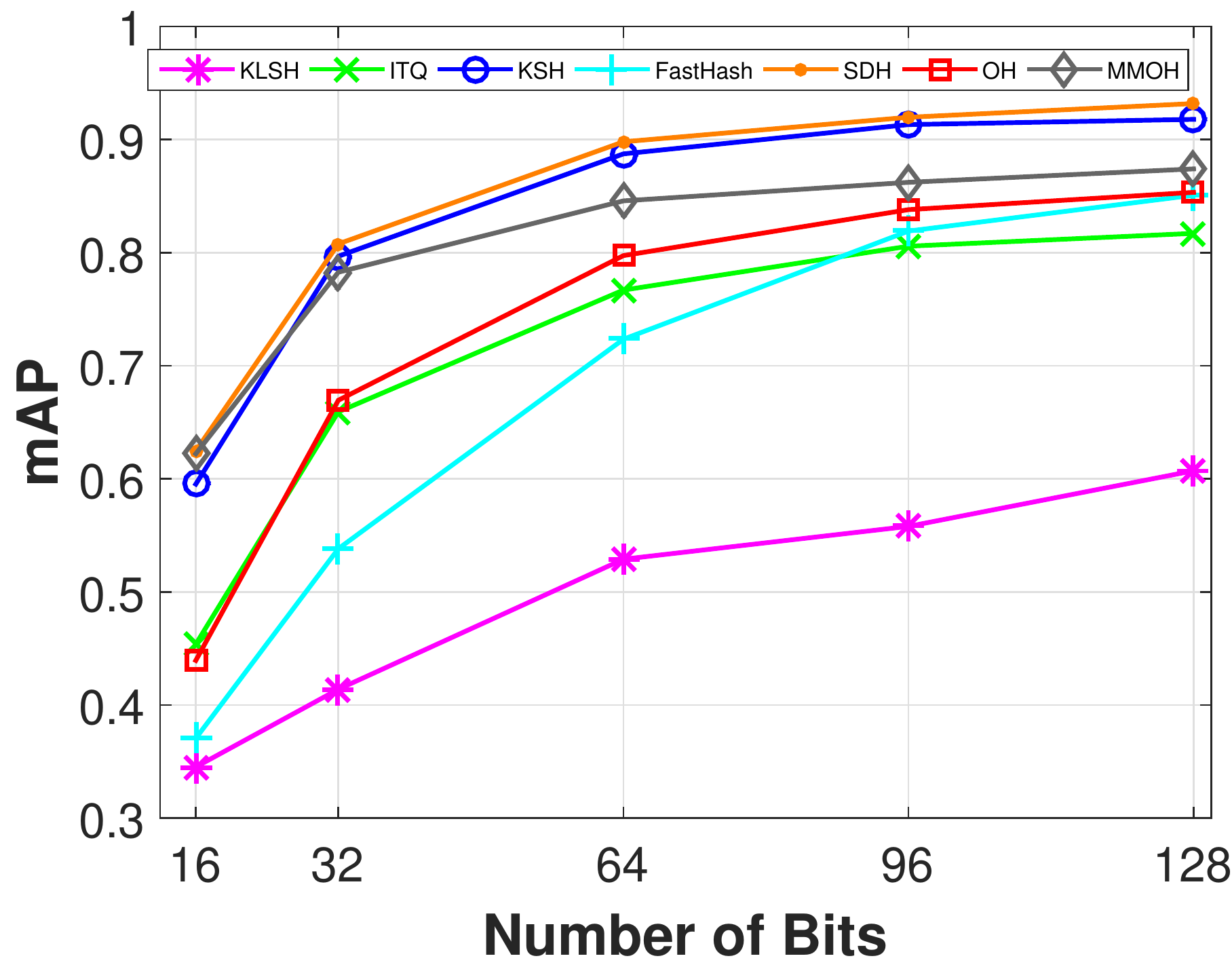}
}
\subfigure[{\scriptsize  22K LabelMe}] 
{
    \label{fig:KLSH_LM}
    \includegraphics[height=0.35 \linewidth]{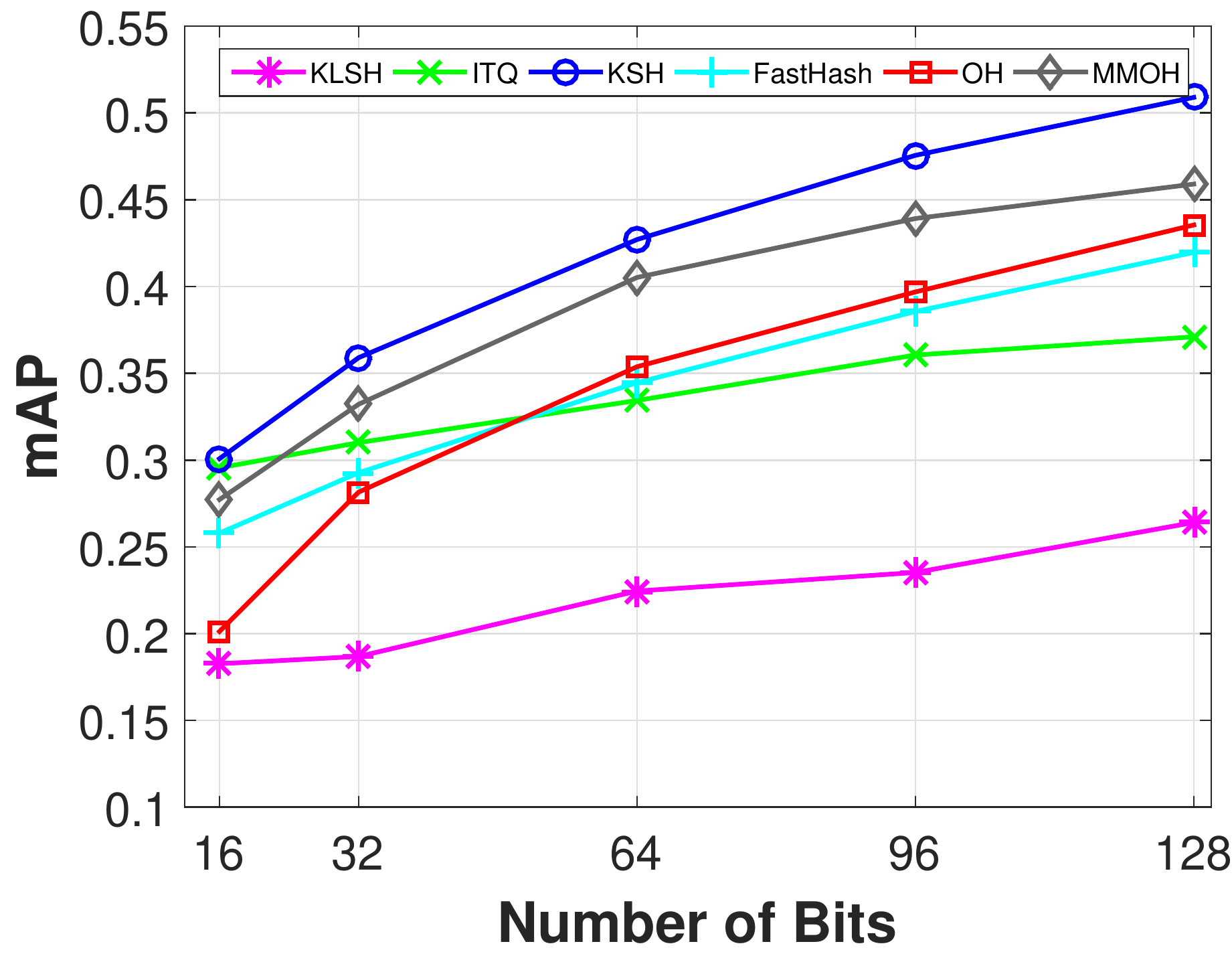}
}
\subfigure[{\scriptsize  GIST 1M}] 
{
    \label{fig:KLSH_gist}
    \includegraphics[height=0.35 \linewidth]{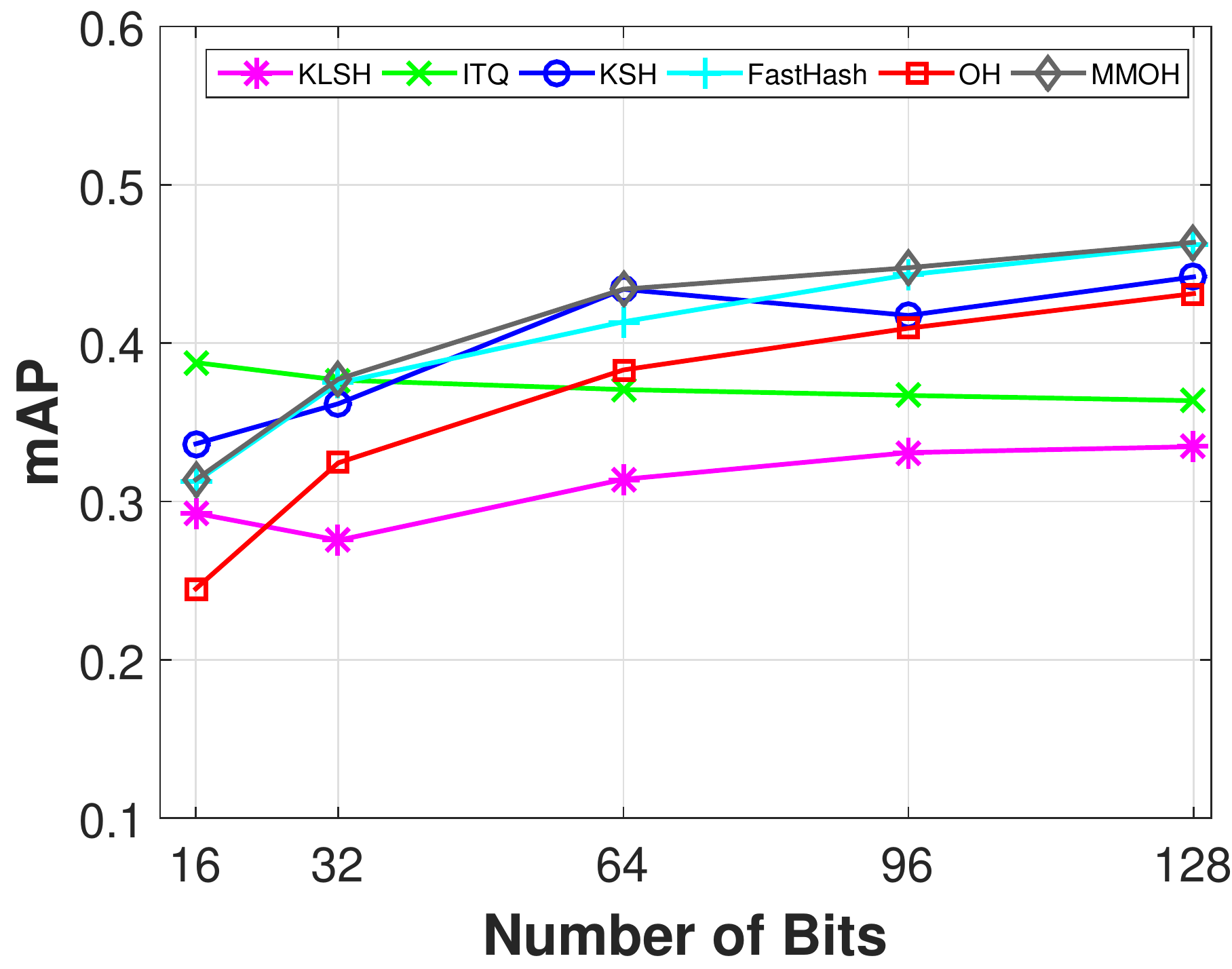}
}
\subfigure[{\scriptsize  CIFAR-10}] 
{
    \label{fig:KLSH_CIFAR}
    \includegraphics[height=0.35 \linewidth]{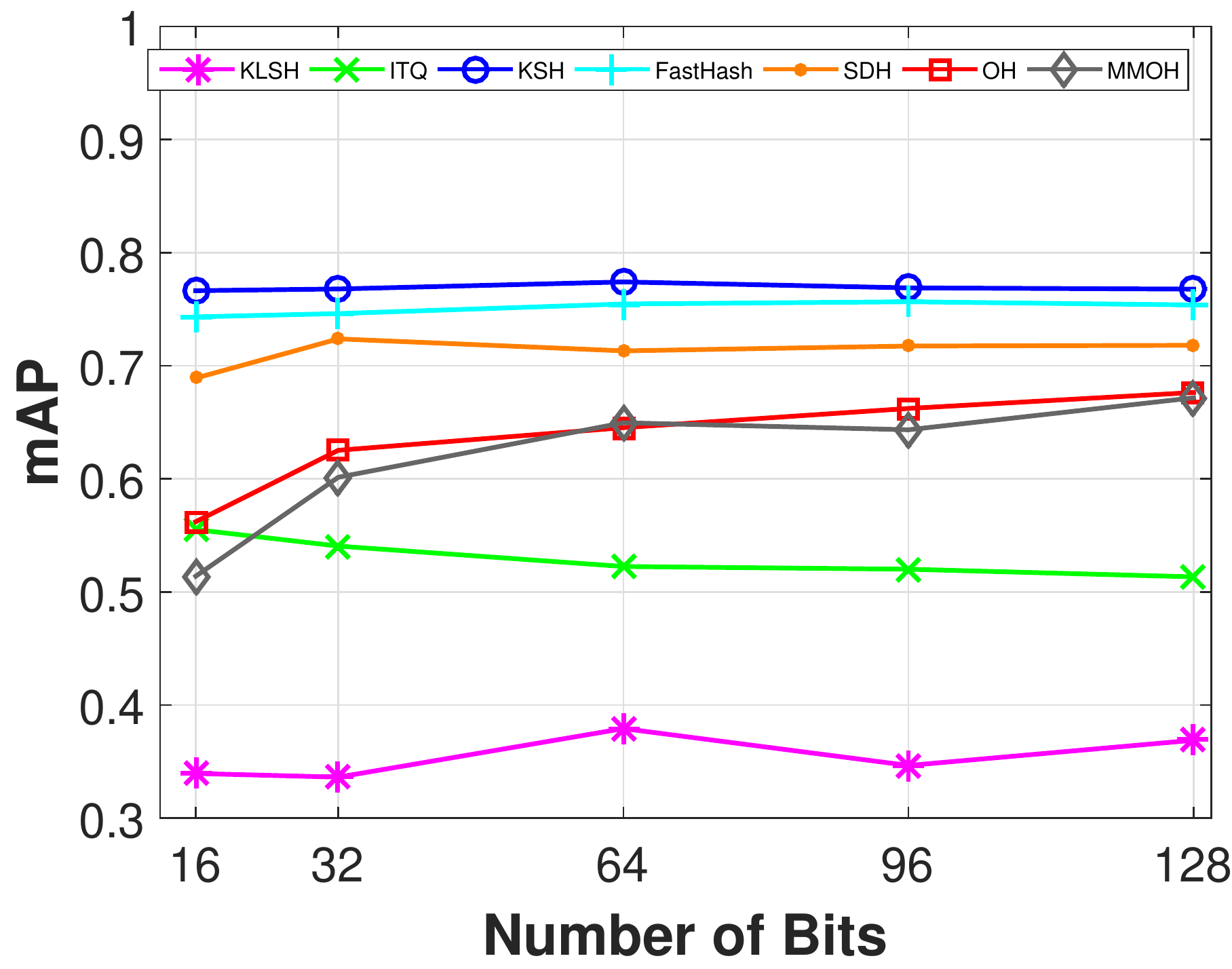}
}
}
\caption{\small mAP comparison against KLSH and learning-based batch mode methods with different code lengths on all datasets. (Best viewed in color.)} 
\label{fig:KLSH-MAP}
\end{figure}

\vspace{-0.2cm}

\section{Conclusion \& Discussion}\label{sec:conclusion}

In this paper, we have proposed a one-pass online learning algorithm for hash function learning called Online Hashing. OH updates its hash model in every step based on a pair of input data samples. We first re-express the hash function as a form of structured prediction and then propose a prediction loss function according to it. By penalizing the loss function using the previously learned model, we update the hash model constrained by an inferred optimal hash codes which achieve zero prediction loss. The proposed online hash function model ensures that the updated hash model is suitable for the current pair of data and has theoretical upper bounds on the similarity loss and the prediction loss. Finally, a multi-model online hashing (MMOH) is developed for a more robust online hashing. Experimental results on different datasets demonstrate that our approach gains satisfactory results, both efficient on training time and effective on the mAP results. \ws{As part of future work, it can be expected to consider updating the model on multiple data pairs at one time. However, the theoretical bound for online learning needs further investigation in this case.}

\vspace{-0.3cm}

\section*{Acknowledgements}

This work was finished when Long-Kai was a undergraduate student at Sun Yat-sen University.

\vspace{-0.3cm}

\bibliographystyle{IEEEtran}
\bibliography{OnlineHashing}

\begin{IEEEbiography}[{\includegraphics[width=1in,height=1.25in,clip,keepaspectratio]{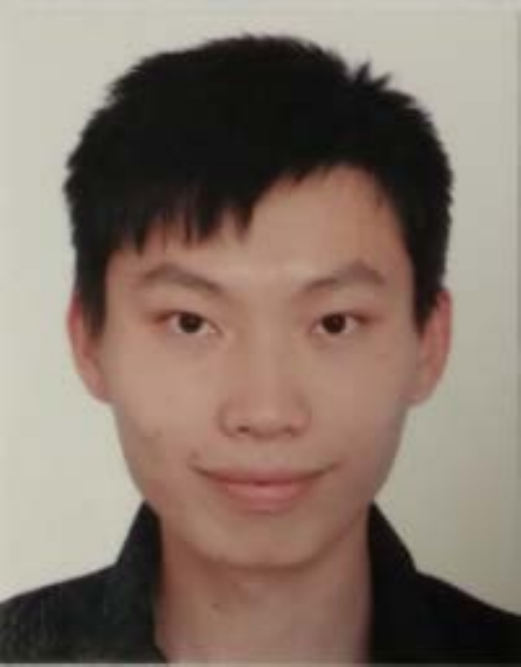}}]{Long-Kai Huang}
is pursuing the Ph.D degree in School of Computer Science and Engineering, Nanyang Technological University, Singapore. He received his B.Eng. degree from School of Information Science and Technology, Sun Yat-sen University, Guangzhou, China in 2013. His currect researh interests are in machine learning and computer vision, and specially focus on fast large-scale image search, non-convex optimization and its applications. 
\end{IEEEbiography}

\begin{IEEEbiography}[{\includegraphics[width=1in,height=1.25in,clip,keepaspectratio]{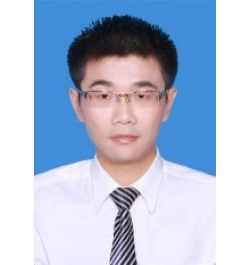}}]{Qiang Yang} 
received his M. S. degree from the School of Infomration Science and Technology, Sun Yat-sen University, China, in 2014. Currently, he is pursuing his Ph. D. degree at the School of Data and Computer Science, Sun Yat-sen University, China. His current research interests include data mining algorithms, machine learning algorithms, evolutionary computation algorithms and their applications on real-world problems.
\end{IEEEbiography}

\begin{IEEEbiography} [{\includegraphics[width=1in,height=1.25in,clip,keepaspectratio]{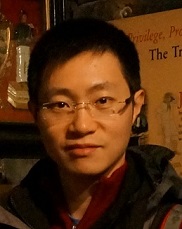}}] {Wei-Shi Zheng} 
is now a Professor at Sun Yat-sen University. He has now published more than 90 papers, including more than 60 publications in main journals (TPAMI,TIP,PR) and top conferences (ICCV, CVPR,IJCAI). His research interests include person/object association and activity understanding in visual surveillance. He has joined Microsoft Research Asia Young Faculty Visiting Programme. He is a recipient of Excellent Young Scientists Fund of the NSFC, and a recipient of Royal Society-Newton Advanced Fellowship. Homepage: http://isee.sysu.edu.cn/\%7ezhwshi/
\end{IEEEbiography}

\end{document}